\documentclass[11pt]{article}

\usepackage[T1]{fontenc}
\usepackage[american]{babel}
\usepackage[margin=1in]{geometry}
\usepackage{setspace}
\usepackage[round]{natbib}
\usepackage{mathtools}
\usepackage{amssymb}
\usepackage{amsthm}
\usepackage{booktabs}
\usepackage{graphicx}
\usepackage[hidelinks]{hyperref}
\usepackage{microtype}
\usepackage{subcaption}

\newtheorem{assumption}{Assumption}
\newtheorem{proposition}{Proposition}
\newtheorem{theorem}{Theorem}
\newtheorem{remark}{Remark}

\newcommand{\E}{\mathbb{E}}
\newcommand{\Var}{\mathrm{Var}}
\newcommand{\Cov}{\mathrm{Cov}}
\newcommand{\op}{\xrightarrow{p}}
\newcommand{\od}{\xrightarrow{d}}

\title{Spatially Robust Inference with Predicted and\\ Missing at Random Labels}

\author{
Stephen Salerno \\
Fred Hutch Cancer Center
\and
Zhenke Wu \\
University of Michigan
\and
Tyler H. McCormick\footnote{Correspondence: tylermc@uw.edu} \\
University of Washington
}
\date{}

\begin{document}

\onehalfspacing
\maketitle

\begin{abstract}
    {\it Inference with predicted data} arises when data are sparsely labeled, but model-based predictions are available for all units and used for statistical inference. While recent methods provide valid uncertainty quantification under independent sampling, real-world applications involve missing at random (MAR) labeling and spatial dependence. For inference in this setting, we propose a doubly robust estimator with cross-fit nuisances. We show that cross-fitting induces fold-level correlation that distorts spatial variance estimators, producing unstable or overly conservative confidence intervals. To address this, we propose a jackknife spatial heteroscedasticity and autocorrelation consistent (HAC) variance correction that separates spatial dependence from fold-induced noise. Under standard identification and dependence conditions, the resulting intervals are asymptotically valid. Simulations and benchmark datasets show substantial improvement in finite-sample calibration, particularly under MAR labeling and clustered sampling.
\end{abstract}

\section{Introduction}

{\it Inference with predicted data} arises when data are sparsely labeled but predicted labels are available and used for statistical inference. This is increasingly common in global health surveillance~\citep{fan2024narratives}, where limited audits are combined with full-coverage prediction maps~\citep{bhatt2015malaria,golding2017child,di2017air}, land-use monitoring, which relies on dense remote-sensing predictions with sparse field verification~\citep{bullock2020amazon,sexton2013treecover}, and citizen-science projects that pair scalable automated classifiers with limited expert labels~\citep{willett2013galaxyzoo2}. Here, the scientific goal is typically not individual-level predictions but valid uncertainty quantification for population-level statistics~\citep{salerno2025we}.

Recent works develop methods for drawing inference with predicted labels~\citep[e.g.,][]{wang2020inference,angelopoulos2023ppi, gronsbell2024another,zrnic2024crossppi}, but these approaches are almost uniformly derived under independent and identically distributed (iid) labeling. Real-world settings routinely violate that assumption in two important ways. First, label acquisition is rarely done at random (i.e., outcome labels are rarely missing completely at random, or MCAR). Labeling commonly depends on observed features and geography, inducing a missing at random (MAR) mechanism. Second, observations are often spatially dependent. Both issues matter for inference. Under MAR, naive estimates can be biased when labeling probabilities are correlated with certain predictors. Further, under spatial dependence, classical variance estimators break down even when point estimators remain approximately unbiased. Addressing both sources of error, selection bias from covariate-dependent labeling and dependence in the data-generating process, is essential for valid inference.

Our framing separates the problem into two components: a {\it base prediction} and a subsequent stage of {\it local nuisance correction}. The base prediction model, $f: \mathcal{X}\to\mathcal{Y}$, produces predicted labels $\widehat{Y}_i = f(X_i)$ for every unit based on predictors, $X_i$. In many applications, this model is trained externally and enters the analysis as a deterministic function. Further, $f$ may be arbitrarily complex, and we do not require any knowledge of its training procedure. For inference, the key point is that $\widehat{Y}$ is generally not a selection-robust substitute for $Y$ under MAR. Correcting both prediction bias and locally varying missingness therefore requires estimating two nuisances using our analytic sample: an outcome model, $\hat m(W_i,s_i)$, and a label propensity model, $\hat\pi(W_i,s_i)$, where $W_i=(X_i,\widehat{Y}_i)$ and $s_i$ are spatial coordinates.

Estimating these nuisance functions in practice is challenging because only a small fraction of units are labeled, making cross-fitting necessary to prevent overfitting~\citep{chernozhukov2018dml}. In spatially dependent settings, however, cross-fitting introduces an unintended consequence: units evaluated using the same nuisance estimates inherit shared estimation noise, creating artificial correlations in the resulting estimator. When standard spatial variance estimators are applied directly, this induced correlation can be mistaken for genuine spatial dependence, leading to unstable or miscalibrated confidence intervals.

We propose a jackknife-based variance correction that separates genuine spatial dependence from artificial correlation introduced by cross-fitting. The method removes fold-induced components from our doubly robust (DR) estimator before applying a Conley-style spatial heteroscedasticity and autocorrelation consistent (HAC) covariance estimator and then restores between-fold variability through an ANOVA-style adjustment~\citep{conley1999gmm}. This jackknife-HAC construction leaves the DR point estimator unchanged, modifying only the variance calculation, making the approach modular and easy to adapt to existing workflows.

Existing methods address individual aspects of this problem. Many approaches for inference with predicted data assume independent sampling~\citep[e.g.,][]{salerno2025moment,angelopoulos2023ppi,zrnic2024crossppi,wang2020inference}. Doubly robust estimators accommodate covariate-dependent labeling in classical missing data and semiparametric settings~\citep{bang2005doubly,zeng2010adjustment,butera2022doubly,little2019missing}. Spatial econometrics provides methods for inference under dependence, including Conley-type covariance estimators and random-field central limit theorems~\citep[CLTs;][]{conley1999gmm,bester2011cluster,jenish2009clt,chandrasekhar2023affinity}. Our contribution is to integrate these ideas into a valid approach and to address a specific issue that arises when nuisance functions are cross-fitted on spatially dependent data. We focus on mean estimation for clarity, the same principles extend to other smooth functionals via influence functions. While our approach addresses spatial dependence with missing outcomes, \citet{chen2025unified} provide a complementary framework for independent data, utilizing pattern-stratified Z-estimation to guarantee valid inference with machine learning imputations under general missing-at-random (MAR) mechanisms across both outcomes and covariates.

Note that {\it externally supplied predictions} correspond to settings where the base prediction model is trained on auxiliary data and then treated as fixed in the analytic sample. When predictions are instead {\it learned from the current dataset}, they are generated out-of-fold to avoid reusing labeled outcomes in both prediction and subsequent correction (see Assumption~\ref{ass:pred}). The proposed jackknife-HAC correction targets a different phenomenon: fold-shared training noise arising from cross-fitted local nuisance estimation (Section~\ref{sec:setup}), rather than dependence induced by fitting the base prediction model itself. Empirical illustrations of this distinction appear in our simulation mechanism (Section~\ref{sec:experiments}) and our mechanism ablation study (Appendix~\ref{app:ablate_tune}).

In the following, Section~\ref{sec:setup} formalizes the estimand and identification under MAR labeling and overlap. Section~\ref{sec:estimation} introduces our estimator, describes cross-fitting the nuisances, and presents the jackknife-HAC variance formula. Section~\ref{sec:asymptotics} states high-level asymptotic normality results under dependent-array CLT conditions. Section~\ref{sec:experiments} reports simulated and real data results, along with diagnostic analyses.

\section{Setup and Identification}
\label{sec:setup}

For units $i=1,\ldots,n$, let $Y_i \in \mathbb{R}$ denote our outcome of interest and $R_i\in\{0,1\}$ indicate whether $Y_i$ is observed. Let $X_i$ collect observed predictors and $s_i\in\mathbb{R}^2$ be a location or spatial dependence coordinate. We assume a predicted outcome, $\widehat{Y}_i$, for every unit is either by {\it externally supplied} or {\it learned from the current dataset} (see Section~\ref{sec:estimation}). For brevity, let $W_i=(X_i,\widehat{Y}_i)$. The observed record is
\[
    O_i = (W_i,s_i,R_i,Y_i^{\rm obs}),
\]
where $Y_i^{\rm obs} = Y_i$ if $R_i = 1$ and is missing if $R_i = 0$.

For simplicity, consider the target as the population mean,
\[
    \theta_0 = \E[Y_i],
\]
though this extends to other smooth functionals. We assume $\E[Y_i]=\theta_0\ \forall i$ and $\E[Y_i^2]<\infty$. Define nuisance functions
\[
\begin{aligned}
    m_0(w,s)&=\E[Y_i\mid W_i=w,s_i=s] \text{ and}\\
    \pi_0(w,s)&=\Pr(R_i=1\mid W_i=w,s_i=s).
\end{aligned}
\]

\begin{assumption}[MAR and Overlap]
\label{ass:mar}
$Y_i \perp R_i \mid (W_i,s_i)$ and there exists $\underline{\pi}>0$ s.t. $\underline{\pi} \le \pi_0(W_i,s_i) \le 1-\underline{\pi}$ a.s.
\end{assumption}

Assumption~\ref{ass:mar} is standard in the missing data literature and allows label availability to depend on observed covariates and spatial location. The overlap condition ensures that every covariate-location stratum contributing to the target mean has a nonzero probability of being labeled. Without overlap, valid inference for the global mean requires either changing the estimand (e.g., restricting to regions with sufficient labeling) or imposing additional assumptions.

Based on \citet{scharfstein1999adjusting}, for any measurable pair $(m,\pi)$, we define the doubly robust estimating function
\[
    \psi_i(\theta; m,\pi) = m(W_i,s_i)+\frac{R_i}{\pi(W_i,s_i)}\{Y_i-m(W_i,s_i)\} - \theta.
\]

\begin{proposition}[Doubly Robust Identification]
\label{prop:dr}
Under Assumption~\ref{ass:mar}, $\E[\psi_i(\theta_0;m,\pi_0)] = 0$ for any measurable $m$, and $\E[\psi_i(\theta_0;m_0,\pi)] = 0$ for any measurable $\pi$ that satisfies the same overlap bound.
\end{proposition}

Proof of Proposition~\ref{prop:dr} can be found in Appendix~\ref{sec:dr_identification}. 

The intuition for $\psi_i$ is that the outcome model term $m_i$ is a calibrated function of the predictions for each unit. The correction term $R_i(Y_i- m_i)/\pi_i$ adds back observed residual information with an inverse-probability of labeling weight so that, under MAR, the labeling bias is removed when either the outcome model is correct or the propensity is correct. Dependence correction enters only through the variance, as point identification is a MAR problem, while calibrated uncertainty is a dependence problem.

The bias of iid methods can be written directly. For exposition only, assume a common marginal label rate, $\bar\pi=\E[R_i]$, and consider one such method~\citep{zrnic2024crossppi}
\[
\begin{aligned}
    \theta_{\mathrm{cppi}}(m)&=\E[m(W_i,s_i)] + \frac{1}{\bar\pi}\E\!\left[R_i\{Y_i-m(W_i,s_i)\}\right] \\
    &=\E[Y_i]+\frac{\Cov\!\left(\pi_0(W_i,s_i), Y_i-m(W_i,s_i)\right)}{\bar\pi}    
\end{aligned}
\]
under MAR labeling. The second term captures selection bias induced by nonrandom labeling: unless the prediction residual, $Y_i-m(W_i,s_i)$, is uncorrelated with the labeling probability, $\pi_0(W_i,s_i)$, there is bias. This representation clarifies why such methods perform well in MCAR settings but fail under MAR, particularly when labeling probabilities vary with features or locations associated with prediction error. Our MAR experiments explicitly construct such dependence by letting label propensities depend on both coordinates and observed covariates.

\section{Estimation and Inference}
\label{sec:estimation}

Our doubly robust estimator solves the sample analog of Proposition~\ref{prop:dr}, but the variance is modified to respect the fold-level noise induced by the cross-fit nuisance functions.

\paragraph{Base Prediction Model and Nuisance Functions.} 
We assume a predicted outcome, $\widehat{Y}_i$, for every unit is produced by a base prediction model that may be trained externally and treated as fixed in the analytic sample. We absorb $\widehat{Y}_i$ into the covariate vector, $W_i=(X_i,\widehat{Y}_i)$. The nuisance functions
\[
\begin{aligned}
    m_0(w,s)&=\E[Y_i\mid W_i=w,s_i=s],\\
    \pi_0(w,s)&=\Pr(R_i=1\mid W_i=w,s_i=s),  
\end{aligned}
\]
must be estimated within the analytic sample to correct local miscalibration and local missingness patterns.

\paragraph{Cross-Fit Local Nuisance Estimation.}
We partition the data into $K$ folds, $\mathcal{I}_1,\ldots,\mathcal{I}_K$. For fold $k$, we estimate $\hat m_{-k}$ and $\hat\pi_{-k}$ on the other $K-1$ folds and evaluate them on $k$:
\[
    \hat m_i=\hat m_{-k}(W_i,s_i),\quad \hat\pi_i=\hat\pi_{-k}(W_i,s_i), \quad i\in \mathcal{I}_k.
\]
To reduce train-test leakage under spatial dependence, we use a buffered set based on distances $d_{ij}=\lVert s_i-s_j\rVert$. Let
\[
\begin{aligned}
    r_n&=\text{$q_b$-quantile}\{d_{ij}:1\le i<j\le n\},\\
    \mathcal{T}_k&=\{i\notin\mathcal{I}_k:\min_{j\in\mathcal{I}_k}d_{ij}>r_n\}.
\end{aligned}
\]
If $|\mathcal{T}_k|$ is too small, we use standard jackknife folds. All propensity predictions are clipped to $[\pi_{\min},1-\pi_{\min}]$ for a fixed $\pi_{\min}\in(0,1/2)$ (in experiments, $\pi_{\min}=0.10$).

\paragraph{DR Point Estimator.}
Define the uncentered DR estimate
\[
    \hat\psi_i^\circ = \hat m_i+\frac{R_i}{\hat\pi_i}(Y_i-\hat m_i), \quad \hat\theta=\frac{1}{n}\sum_{i=1}^n \hat\psi_i^\circ.
\]

\paragraph{Naive Spatial HAC Instability.}
Cross-fitting prevents overfitting the labeled subset, but it has a side effect under spatial dependence: all units in a fold, $k$, share the same fitted nuisance functions, $(\hat m_{-k},\hat\pi_{-k})$, so their estimates $\hat\psi_i^\circ$ share {\it fold-level noise}. A stylized decomposition is
\[
    \hat\psi_i^\circ=\theta_0+\phi_{0i}+a_{k(i)}+\varepsilon_{n,i},
\]
where $\phi_{0i}=\psi_i(\theta_0;m_0,\pi_0)$ is the centered score function, $a_{k(i)}$ is a fold-level random effect induced by nuisance estimation in settings of spatial dependence, and $\varepsilon_{n,i}$ is a remainder. For $i$ and $j$ in the same fold, $\Cov(\hat\psi_i^\circ,\hat\psi_j^\circ)$ includes $\Var(a_{k(i)})$ regardless of distance. Since a nontrivial fraction of near-neighbor pairs lie in the same fold, a distance-weighted Conley HAC estimator can misattribute this shared nuisance noise to genuine short-range spatial dependence, inflating standard errors in finite samples.

\paragraph{Jackknife-HAC Variance.}
Let $\bar\psi_k=n_k^{-1}\sum_{i\in\mathcal{I}_k}\hat\psi_i^\circ$, $n_k=|\mathcal{I}_k|$, and define fold-centered scores $\tilde\psi_i=\hat\psi_i^\circ-\bar\psi_{k(i)}$. Let $\kappa(\cdot)$ be a triangular kernel and set the HAC bandwidth $h_n$ by a distance quantile:
\[
\begin{aligned}
    \kappa(u)&=\max(1-u,0),\\
    h_n&=\text{$h_q$-quantile}\{d_{ij}:1\le i<j\le n\}.
\end{aligned}
\]
In experiments we choose $h_n$ (and the buffer radius $r_n$) by distance quantiles, $h_q$, for simplicity and reproducibility. In the asymptotic statement, we treat these as deterministic sequences that satisfy standard spatial HAC and buffered-splitting regularity conditions. Define weights $w_{ij}=\kappa(d_{ij}/h_n)$. The within-fold covariance is a Conley HAC quadratic form applied to $\tilde\psi$:
\[
    \hat V_{\mathrm{within}} =\frac{1}{n^2}\sum_{i=1}^n\sum_{j=1}^n w_{ij}\tilde\psi_i\tilde\psi_j.
\]
We keep only the off-diagonal contribution and add the between-fold term computed from the fold means:
\[
    \hat V_{\mathrm{off}} =\hat V_{\mathrm{within}}-\hat V_{\mathrm{diag}}, \quad \hat V_{\mathrm{diag}}=\frac{1}{n^2}\sum_{i=1}^n \tilde\psi_i^2,
\]
\[
    \hat V_{\mathrm{between}} =\frac{K}{K-1}\sum_{k=1}^K\left(\frac{n_k}{n}\right)^2(\bar\psi_k-\hat\theta)^2.
\]
The jackknife-HAC variance estimator is
\[
\hat V_{\mathrm{JK}}
=\hat V_{\mathrm{off}}+\hat V_{\mathrm{between}}.
\]
Intuitively, fold centering removes fold-shared components such as $a_{k(i)}$ from the spatial covariance calculation, while the between-fold term reintroduces fold-mean variation only through the fold averages. The diagonal subtraction prevents counting idiosyncratic (diagonal) variation twice. In our implementation, we further let $\hat V_{\mathrm{JK}}^{+}=\max(\hat V_{\mathrm{JK}},\epsilon)$ with small $\epsilon>0$ for numerical stability. A two-sided $1-\alpha$ confidence interval is then given by
\[
    \hat\theta\pm c_{1-\alpha/2}\sqrt{\hat V_{\mathrm{JK}}^{+}},
\]
where $c_{1-\alpha/2}=z_{1-\alpha/2}$ in the spatial branch.

\begin{proposition}[Fold-Shared Noise Is Removed from Within Covariance]
\label{prop:fold_demean}
Fix a fold partition $\{\mathcal{I}_k\}_{k=1}^K$. Suppose the computed score functions admit a decomposition
\[
    \hat\psi_i^\circ = u_i + a_{k(i)},
\]
where $a_{k(i)}$ is constant within fold $k(i)$. Let $\tilde\psi_i=\hat\psi_i^\circ-\bar\psi_{k(i)}$ be the centered fold score and $\tilde u_i=u_i-\bar u_{k(i)}$ the fold-centered $u$-sequence. Then $\tilde\psi_i=\tilde u_i$ for all $i$, so any within-fold covariance applied to $\tilde\psi$, including the Conley HAC quadratic form, is invariant to adding fold-constant terms $\{a_k\}$. Consequently, fold-shared noise affects $\hat V_{\mathrm{JK}}$ only through $\hat V_{\mathrm{between}}$, computed from the fold means.
\end{proposition}

\begin{proof}
Write $\bar\psi_k=n_k^{-1}\sum_{i\in\mathcal{I}_k}\hat\psi_i^\circ=\bar u_k+a_k$. Then for any $i\in\mathcal{I}_k$,
\[
\tilde\psi_i
=\hat\psi_i^\circ-\bar\psi_k
=(u_i+a_k)-(\bar u_k+a_k)
=u_i-\bar u_k
=\tilde u_i.
\]
\end{proof}

\paragraph{Optional Moran Gate.}
In implementation, we optionally gate the within-fold covariate by a labeled residual Moran test: compute Moran's $I$~\citep{moran1950} on $\{Y_i-\hat m_i:R_i=1\}$. If the test does not reject a spatial signal, we replace the within-fold HAC covariance by an iid covariance matrix and use a Student's $t_{K-1}$ critical value for $c_{1-\alpha/2}$. This gate is a pragmatic stability heuristic. Our analytic results treat $\hat V_{\mathrm{JK}}^{+}$ abstractly through a high-level consistency assumption.

\paragraph{Algorithm Summary (One Replicate).}
Our procedure is as follows: (i) obtain the base predicted outcomes, $\widehat{Y}_i$; (ii) choose folds and buffer, $q_b$, then cross-fit local nuisance models, $\hat m(W,s)$ and $\hat\pi(W,s)$, on the analytic sample; (iii) compute $\hat\theta=n^{-1}\sum_i \hat\psi_i^\circ$; (iv) compute $\hat V_{\mathrm{JK}}^{+}$; and (v) report
\[
    \hat\theta\pm c\sqrt{\hat V_{\mathrm{JK}}^{+}}.
\]

\paragraph{Asymptotic Fold Scaling.}
For the high-level asymptotic statement in Section~\ref{sec:asymptotics} we allow the number of folds $K=K_n$ to increase with $n$ (Assumption~\ref{ass:folds}), which makes the ANOVA term stable in large samples. In finite samples we use small $K$ for feasibility with limited labels.

Two estimand conventions should be distinguished in finite-sample reporting. Our analytic results target the superpopulation mean, $\theta_0=\E[Y_i]$, which treats $(W_i,s_i,Y_i)$ as random draws from a dependent array. Here, our main results use superpopulation-style coverage. In synthetic designs, this is for a known data generation process mean, $\mu$, but in observed benchmark resampling, this is a fixed full-pool mean used across repeated draws. Finite-target coverage against each realized sample mean is reported only as a comparison in the Appendix, where it is typically more attenuated. This convention follows the finite-versus-superpopulation distinction in modern causal inference~\citep{jin2024tailored} and is conceptually parallel to the fixed-versus-random spatial structure distinction emphasized in spatial confounding~\citep{gilbert2025consistency}.

\begin{remark}[Buffered Cross-Fitting]
Buffered cross-fitting is important in dependent settings. Standard random cross-fitting in these settings can leak local dependence from between folds. Buffering reduces this leakage by excluding observations within a distance quantile around each held-out fold. In our experiments we use small buffers so that nuisance estimation remains stable while reducing nearest-neighbor dependence leakage. This is a pragmatic compromise between asymptotic orthogonality arguments and finite-sample feasibility when the labeled fraction is small.   
\end{remark}

\begin{remark}[Propensity Clipping]
$\pi_{\min}$, is tied to the overlap assumption and the empirical design. $\pi_{\min}=0.10$ reflects an explicit two-sided overlap floor in the MAR labeling mechanism and ensures bounded influence weights. In settings where audit budgets imply weaker overlap, the same framework can be used with a different target estimand, such as a trimmed overlap population. We keep the global-mean estimand for our main results and treat no-overlap designs as supplementary because $\theta_0$ is not point identifiable without additional structure.  
\end{remark}

\section{Asymptotic Validity Under Dependence}
\label{sec:asymptotics}

Let $\phi_{0i}=\psi_i(\theta_0;m_0,\pi_0)$ and $Z_{n,i}=\phi_{0i}-\E[\phi_{0i}]$. Our asymptotic interpretation follows increasing-domain spatial asymptotics: as $n\to\infty$, the spatial domain expands, local dependence neighborhoods remain bounded in physical scale while long-range covariance accumulation remains weak enough for HAC stabilization.

Let $\mathcal{A}_n(i)\subset\{1,\ldots,n\}$ be affinity sets with $i\in\mathcal{A}_n(i)$, instantiated as distance neighborhoods $\mathcal{A}_n(i)=\{j:d_{ij}\le h_n\}$, and define
\[
    \Omega_n=\sum_{i=1}^n\sum_{j\in\mathcal{A}_n(i)}\Cov(Z_{n,i},Z_{n,j}).
\]
We take $h_n$ to be the same bandwidth sequence used by the spatial covariance in Section~\ref{sec:estimation}.

\begin{assumption}[Covariance CLT Conditions]
\label{ass:clt}
The within-affinity covariance accumulation, cross-affinity negligibility, and outside-affinity tail controls in \citet{chandrasekhar2023affinity} hold for $\{Z_{n,i}\}_{i=1}^n$, and
\[
    \Omega_n/n \to \sigma^2 \in (0,\infty).
\]
\end{assumption}

\begin{assumption}[Fold Regularity]
\label{ass:folds}
Let $K=K_n$ be the number of cross-fitting folds with partition $\{\mathcal{I}_k\}_{k=1}^{K_n}$. As $n\to\infty$, $K_n\to\infty$, $\min_k|\mathcal{I}_k|\to\infty$, and $\max_k|\mathcal{I}_k|/n\to 0$. Fold assignments are fixed a priori or generated by auxiliary randomness independent of $\{O_i\}_{i=1}^n$ conditional on $\{(W_i,s_i)\}_{i=1}^n$, and the same folds are used in estimating the nuisance functions and in estimating $\hat V_{\mathrm{JK}}$.
\end{assumption}

\begin{assumption}[Predicted Outcome Provenance]
\label{ass:pred}
Either (i) predicted outcomes, $\widehat{Y}_i$, are produced by a model trained on auxiliary data independent of the current analysis sample $\{O_i\}_{i=1}^n$ (so conditional on that auxiliary training data, the prediction rule is fixed), or (ii) if the predictions are generated from the current dataset, the prediction generation stage is cross-fitted (with optional buffering) so that for each labeled unit $i$, $\widehat{Y}_i$ is out-of-fold with respect to its own observed outcome $Y_i$. In all cases, any additional calibration step used to form $\hat m(W,s)$ is also cross-fitted so that $\hat m_i$ is out-of-fold for its labeled units.
\end{assumption}

\begin{assumption}[Cross-fit Nuisance Quality]
\label{ass:nuisance}
The cross-fitted nuisance functions satisfy
\[
    \frac{1}{n}\sum_{i=1}^n\Big(\psi_i(\theta_0;\hat m,\hat\pi)-\psi_i(\theta_0;m_0,\pi_0)\Big)=o_p(n^{-1/2}),
\]
and $\hat\pi_i\in[\underline{\pi},1-\underline{\pi}]$ with probability approaching one.
\end{assumption}

\begin{assumption}[Jackknife-HAC Consistency]
\label{ass:hac}
$\hat V_{\mathrm{JK}}^{+}\ge 0$ and $n\hat V_{\mathrm{JK}}^{+} \op \sigma^2$.
\end{assumption}

Assumption~\ref{ass:folds} allows the number of folds $K_n$ to grow while each fold remains large, which stabilizes the between-fold term in $\hat V_{\mathrm{JK}}$. Assumption~\ref{ass:pred} separates external- versus within-sample prediction regimes; when predicted outcomes are generated from current labels, cross-fitting is needed to avoid own-observation reuse. Assumption~\ref{ass:nuisance} is a high-level orthogonality remainder condition standard in debiased machine learning with cross-fitting \citep{chernozhukov2018dml}.  In iid settings, cross-fitting plus bounded moments and the DML product-rate condition
\[
    \Big(\frac{1}{n}\sum_{i=1}^n(\hat m_i-m_{0i})^2\Big)^\frac{1}{2}\Big(\frac{1}{n}\sum_{i=1}^n(\hat\pi_i-\pi_{0i})^2\Big)^\frac{1}{2}=o_p(n^{-\frac{1}{2}})
\]
is sufficient for Assumption~\ref{ass:nuisance} \citep{chernozhukov2018dml}. Under dependence, we impose Assumption~\ref{ass:nuisance} directly rather than deriving it from primitive mixing conditions. Importantly, the identification-level double robustness in Proposition~\ref{prop:dr} does not by itself imply Assumption~\ref{ass:nuisance}: asymptotic normality requires nuisance errors to be sufficiently small in the sense of the remainder bound above. Assumption~\ref{ass:hac} concerns the jackknife-HAC variance estimator; sufficient conditions combine standard kernel, bandwidth, and covariance-decay conditions for Conley HAC \citep{conley1999gmm,bester2011cluster} with fold regularity and the requirement that centering the folds removes only asymptotically negligible fold-shared nuisance noise.

\begin{theorem}[Asymptotic Normality and Valid CIs]
\label{thm:main}
Under Assumptions~\ref{ass:mar}--\ref{ass:hac},
\[
    \sqrt{n}(\hat\theta-\theta_0)\od \mathcal{N}(0,\sigma^2),
\]
and therefore
\[
    \Pr\left(\theta_0\in\left[\hat\theta-z_{1-\frac{\alpha}{2}}\sqrt{\hat V_{\mathrm{JK}}^{+}}, \hat\theta+z_{1-\frac{\alpha}{2}}\sqrt{\hat V_{\mathrm{JK}}^{+}}\right]\right)\to 1-\alpha.
\]
\end{theorem}

\begin{proof}
Write
\[
    \sqrt{n}(\hat\theta-\theta_0) =\frac{1}{\sqrt{n}}\sum_{i=1}^n \phi_{0i} +\sqrt{n} R_n,
\]
where
\[
    R_n=\frac{1}{n}\sum_{i=1}^n\Big(\psi_i(\theta_0;\hat m,\hat\pi)-\psi_i(\theta_0;m_0,\pi_0)\Big).
\]
By Assumption~\ref{ass:nuisance}, $\sqrt{n}R_n=o_p(1)$. By Assumption~\ref{ass:clt}, the centered leading term obeys
\[
    \frac{1}{\sqrt{n}}\sum_{i=1}^n Z_{n,i}\od \mathcal{N}(0,\sigma^2).
\]
Because
\[
    \frac{1}{\sqrt{n}}\sum_{i=1}^n\phi_{0i}=\frac{1}{\sqrt{n}}\sum_{i=1}^n Z_{n,i}+\frac{1}{\sqrt{n}}\sum_{i=1}^n \E[\phi_{0i}],
\]
and $\E[\phi_{0i}]=0$ for each $i$ by Proposition~\ref{prop:dr}, we obtain
\[
    \frac{1}{\sqrt{n}}\sum_{i=1}^n\phi_{0i}\od \mathcal{N}(0,\sigma^2).
\]
Slutsky then yields the first claim. The interval claim follows from Assumption~\ref{ass:hac} and another Slutsky step.
\end{proof}

\begin{remark}[Scope of Theorem~\ref{thm:main}]
The theorem is conditional on Assumptions~\ref{ass:clt}--\ref{ass:hac}. In particular, we do not claim that finite samples can directly verify every covariance-array and nuisance-rate condition, especially when coordinates are proxy embeddings. The next section should therefore be read as assumption-aligned empirical support, rather than proof that each benchmark dataset satisfies all asymptotic regularity conditions.
Further, the data-dependent Moran pre-test described in Section~\ref{sec:estimation} selects between variance/critical-value branches and is not covered by Theorem~\ref{thm:main}; the theorem treats $\hat V_{\mathrm{JK}}^{+}$ as the variance.
\end{remark}

For the complete proof, see Appendix~\ref{app:proof}.

\section{Comparator Methods and Empirical Results}
\label{sec:experiments}

This section addresses one question: under MAR or MCAR designs with sufficient overlap, does dependence-aware inference with predicted labels improve coverage relative to iid baselines, and how does that pattern change as dependence and sampling concentration increase?

We compared our approach (referred to as Spatial DR-JK-HAC) to the following methods: Cross-PPI~\citep{zrnic2024crossppi}, PPI++~\citep{angelopoulos2023ppi++}, and Bootstrap-PPI~\citep{efron2026boostrap}. These serve as example baselines, as they assume independent labeling and iid sampling. Spatial DR-JK-HAC is our doubly robust point estimator paired with a jackknife-HAC variance designed for MAR labeling and spatial dependence. Additional DR baselines (DR-iid and spatial DR-HAC) are reported in the Appendix as robustness comparisons; they follow the same DR point estimator but pair it with an iid variance or a direct Conley-style HAC variance, respectively.

\subsection{Simulation Mechanism and Results}
We generate a synthetic spatial population on an $n_{\mathrm{side}}\times n_{\mathrm{side}}$ grid (here $n_{\mathrm{side}}=250$, so $N=62{,}500$ population units). Each population replicate draws smoothed Gaussian random fields on this grid. The dependence level $\sigma$ is the smoothing radius (in grid cells); larger $\sigma$ induces longer-range spatial correlation. To keep the marginal scale comparable across $\sigma$, we center and rescale each $\sigma$-smoothed field to have approximately zero mean and unit standard deviation.

Each unit has observed features $W=(X,U_{\mathrm{obs}},\mathrm{coords})$, where $X$ is a weakly dependent field (fixed smoothing radius $2$) and $\mathrm{coords}$ are normalized coordinates. Outcomes follow
\[
Y = \mu + \beta X + \lambda U_{\mathrm{obs}} + \lambda U_{\mathrm{unobs}} + \varepsilon,
\]
with $\mu=0$, $\beta=0.8$, $\lambda=1.0$, and $\varepsilon\sim\mathcal{N}(0,0.6^2)$. $U_{\mathrm{unobs}}$ is a second $\sigma$-smoothed field not included in $W$ to create spatial dependence that cannot be removed by conditioning on $W$, so predictions trained on $W$ are intentionally imperfect at high dependence.

Predictions, $\hat Y$, are based on an auxiliary training set \citep[see Remark~5.1 of][]{xu2025unifiedframeworksemiparametricallyefficient}. Throughout, we use an {\it auxiliary-split} construction: we split the population into an auxiliary training pool (35\%) and an analysis pool (65\%), fit a gradient-boosted tree predictor of $Y$ from standardized $W$ on the auxiliary pool, and evaluate $\hat Y$ only on the disjoint analysis pool. We then draw repeated analysis samples of size $n=600$ from the analysis pool under either iid sampling or soft-block sampling. In soft-block sampling, a small ``core'' of size $0.05n$ is taken as nearest neighbors of a random anchor point, and the remainder is filled iid from the remaining pool, making dependence more visible to the sampling distribution, but without making global superpopulation inference impossible.

Given a drawn sample, we take a fixed label budget of 20\% either MCAR or via an adversarial MAR mechanism that concentrates labels away from high-uncertainty regions while respecting $\pi\in[0.10,0.90]$ (see Appendix~\ref{sec:propensity_dgm}). To keep the simulation focused on dependence-driven calibration, we use the true $\pi$ when forming the DR estimator. 

Figure~\ref{fig:synth_cov_box} reports the 90\% coverage probabilities for $\sigma\in\{0,40,80,120\}$ (100 population replicates per $\sigma$, 200 sample draws per population). Under iid sampling and MCAR labeling, the baseline methods have near nominal coverage, even as $\sigma$ grows (e.g., at $\sigma=120$, Cross-PPI and PPI++ cover at $0.90$ and $0.90$). Under iid sampling and MAR, these methods under-cover (at $\sigma=120$, Cross-PPI: $0.46$, PPI++: $0.57$, Bootstrap-PPI: $0.33$), reflecting MAR-induced bias. Under soft-block sampling, these methods under-cover as dependence grows, even in MCAR (at $\sigma=120$, Cross-PPI: $0.81$, PPI++: $0.80$, Bootstrap-PPI: $0.65$), because spatial dependence is visible in the sampled data. Spatial DR-JK-HAC preserves near-nominal coverage across all four cells (at $\sigma=120$, block-MCAR: $0.91$, block-MAR: $0.90$).

Appendix~\ref{app:synthwidth} summarizes tradeoffs in CI widths. At $\sigma=120$ under soft-block sampling, Spatial DR-JK-HAC intervals are $1.36\times$ wider than Cross-PPI in MCAR ($0.39$ vs $0.28$ mean length) and $1.57\times$ wider in MAR ($0.43$ vs $0.27$).

\begin{figure}[t]
\centering
\includegraphics[width=\textwidth]{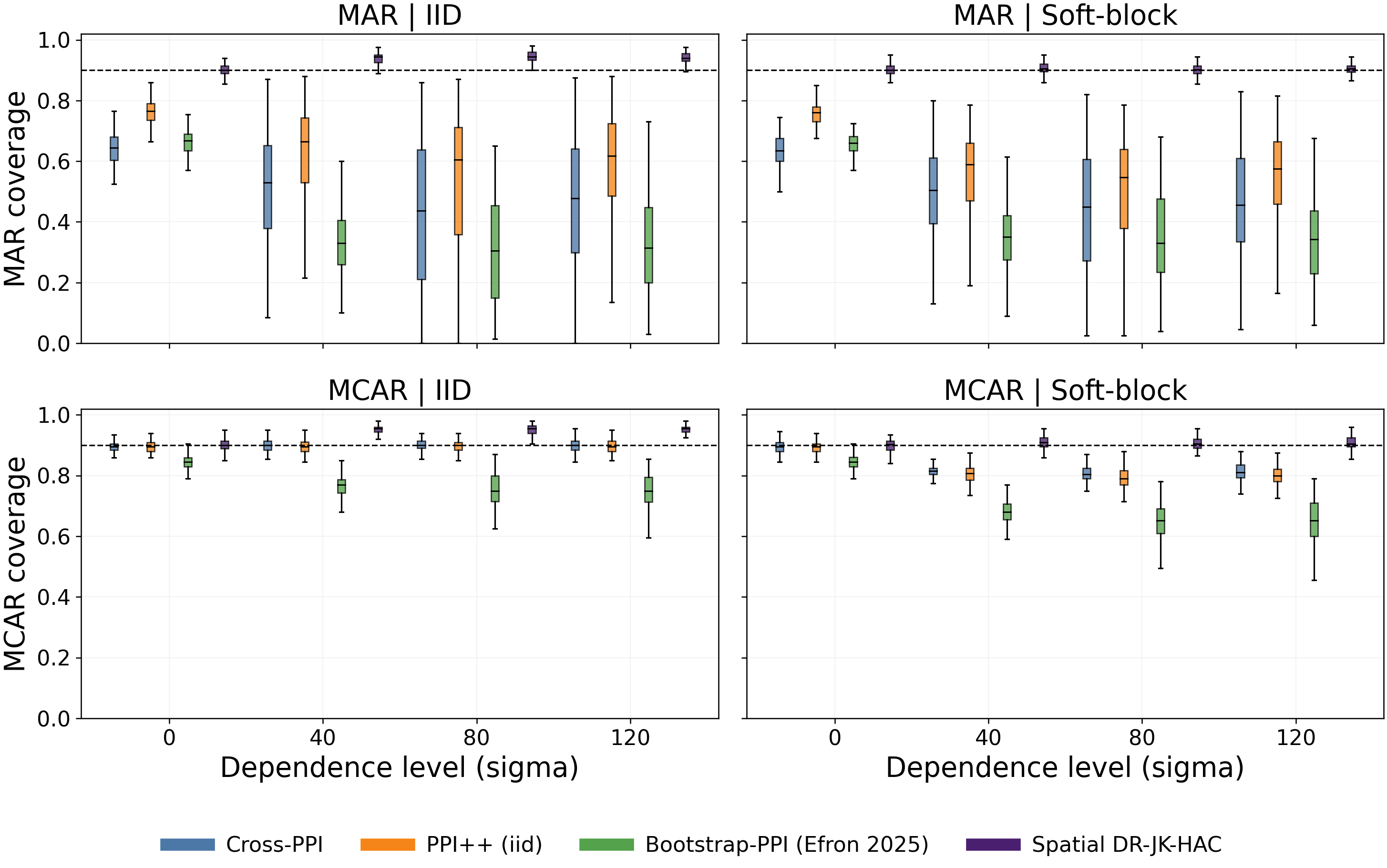}
\caption{Synthetic spatial-field simulation (superpopulation target): boxplots of empirical 90\% coverage versus dependence level $\sigma$. Panels split by missingness arm (MCAR vs MAR) and sampling regime (iid vs soft-block). For each $\sigma$, each box summarizes the distribution across 100 population replicates of empirical coverage computed from 200 repeated sample draws of size $n=600$ from a single population replicate. Methods shown are Cross-PPI, PPI++, Bootstrap-PPI, and Spatial DR-JK-HAC. The dashed line marks nominal coverage. Under MAR, baseline models under-cover even under iid sampling; under MCAR, these models under-cover when soft-block sampling makes spatial dependence visible. Spatial DR-JK-HAC stays near nominal in all four cells, at the cost of modestly wider intervals (see Appendix~\ref{app:synthwidth}).}
\label{fig:synth_cov_box}
\end{figure}

\subsection{Observed Data Design and Results}

We evaluated five observed benchmark datasets, three from \citet{zrnic2024crossppi}: (i) {\it forest disturbance}, a geolocated Amazon deforestation dataset with remote-sensing covariates~\citep{bullock2020amazon,sexton2013treecover}; (ii) {\it galaxy zoo morphology}, based on Galaxy Zoo 2 labels with associated SDSS imaging metadata~\citep{willett2013galaxyzoo2,york2000sdss}; and (iii) {\it census income}, constructed from the American Community Survey via Folktables~\citep{ding2021folktables}, and two harder spatial benchmarks: (iv) {\it malaria} (Colombian municipalities; outcome is a transformed malaria burden variable), and (v) {\it Health+} (a tract-level life-expectancy outcome with socioeconomic covariates). In each dataset, we repeatedly draw $n=320$ observations under iid and soft-block sampling and mask labels under MCAR or MAR at $20\%$ budget. Coverage uses the fixed full-pool benchmark mean as the repeated-sampling target (the observed-data analog of a superpopulation-style target in this design). All observed-data results use the same auxiliary-split prediction construction described above, with 40 replicates per dataset-arm-sampling cell for $K=5$ methods.


Table~\ref{tab:real-baux-main-coverage} summarizes coverage for the observed data under this auxiliary-split prediction design, averaged over iid and soft-block sampling. Averaged across datasets in MAR, coverage is $0.665$ (Cross-PPI), $0.726$ (PPI++), $0.571$ (Bootstrap-PPI), $0.810$ (DR-iid), $0.816$ (Spatial DR-HAC), and $0.874$ (Spatial DR-JK-HAC, $K=5$). Forest and Malaria show the largest MAR gains for Spatial DR-JK-HAC ($K=5$, $t$) relative to Spatial DR-HAC. In MCAR, DR-based methods are generally near or above nominal, while Bootstrap-PPI remains low in several datasets.

CI-width differences are directionally consistent with the synthetic design. In MAR settings, Spatial DR-JK-HAC (K=5, $t$) has about $1.38\times$ the mean width of Spatial DR-HAC when averaging widths across datasets. Appendix~\ref{app:baux_tables} reports K=10 fold-sensitivity and width tables: K=10 ($t$) reduces width inflation relative to K=5 while retaining higher coverage than Spatial DR-HAC in most datasets, and K=10 ($z$) is close to Spatial DR-HAC in both coverage and width.

\begin{table}[t]
\centering
\caption{Coverage for observed data auxiliary-split prediction design, averaged over IID and soft-block sampling.}
\label{tab:real-baux-main-coverage}
\setlength{\tabcolsep}{4.0pt}
\small
\resizebox{\linewidth}{!}{%
\begin{tabular}{lrrrrrr}
\toprule
Dataset & Cross-PPI & PPI++ (IID) & Bootstrap-PPI & DR-IID (CF-m) & Spatial DR-HAC & Spatial DR-JK-HAC (K=5, \(t\)) \\
\midrule
\multicolumn{7}{l}{\textit{Panel: MAR}} \\
Forest & 0.613 & 0.619 & 0.575 & 0.825 & 0.825 & 0.925 \\
Galaxies & 0.775 & 0.875 & 0.825 & 0.825 & 0.825 & 0.875 \\
Census income & 0.631 & 0.831 & 0.613 & 0.725 & 0.744 & 0.806 \\
Malaria & 0.850 & 0.900 & 0.725 & 0.750 & 0.750 & 0.850 \\
Health+ & 0.456 & 0.406 & 0.119 & 0.925 & 0.938 & 0.912 \\
\midrule
\multicolumn{7}{l}{\textit{Panel: MCAR}} \\
Forest & 0.894 & 0.856 & 0.794 & 0.900 & 0.894 & 0.969 \\
Galaxies & 0.912 & 0.919 & 0.825 & 0.906 & 0.906 & 0.944 \\
Census income & 0.863 & 0.844 & 0.713 & 0.850 & 0.875 & 0.906 \\
Malaria & 0.850 & 0.838 & 0.738 & 0.906 & 0.906 & 0.944 \\
Health+ & 0.912 & 0.906 & 0.263 & 0.931 & 0.944 & 0.900 \\
\bottomrule
\end{tabular}
}
\vspace{2pt}
\begin{minipage}{0.98\linewidth}
\footnotesize Each entry averages over both sampling regimes for the given dataset and arm. All entries use nominal 90\% intervals.
\end{minipage}
\end{table}

We also ran targeted diagnostic checks to isolate mechanisms and rule out obvious analysis artifacts. First, we validate coordinate proxies by residual Moran and nearest-neighbor diagnostics rather than narrative caveats, and we report that proxy quality is strongest for forest and weaker for galaxies and census (Appendix~\ref{app:proxy_overlap}). Second, we report overlap diagnostics directly, including design and estimated propensity quantiles, clipping rates, and inverse-propensity effective sample size (Appendix~\ref{app:proxy_overlap}). Third, we provide a mechanism attribution check that isolates the roles of DR correction, dependence-aware variance estimation, and the jackknife-HAC correction in the hardest synthetic cell (Appendix~\ref{app:ablate_tune}). Replicate-level interval visualizations that make MAR centering failures visible are shown in Appendix~\ref{app:ci_examples}.

\section{Scope, Novelty, and Limits}

Under Assumption~\ref{ass:mar}, the point estimator in this paper is deliberately classical: it is the standard doubly robust (DR) estimator for missing outcomes~\citep{bang2005doubly,little2019missing}. The novelty is in the {\it covariance correction} required once the DR nuisance functions are estimated locally in the analytic sample with cross-fitting. Existing methods are built for independent designs and typically do not address MAR labeling plus spatial dependence in one end-to-end pipeline~\citep{angelopoulos2023ppi++,zrnic2024crossppi}. Spatial HAC estimators are well developed~\citep{conley1999gmm,bester2011cluster}, but applying them directly to cross-fitted DR scores can be unstable because fold members share the same fitted nuisance functions.

This paper isolates and addresses these issues. Cross-fitting is essential with small label budgets, but it induces fold-shared training noise in $\hat m$ and $\hat\pi$, hence, in the estimated score functions, $\hat\psi_i^\circ$. When nearby units share a fold, this noise can masquerade as short-range spatial correlation and inflate a Conley HAC variance estimate. The jackknife-HAC variance in Section~\ref{sec:estimation} attenuates this fold-shared component by centering within folds and recombining within- and between-fold variation. This correction leaves the identification score unchanged and targets the variance estimation.

Our assumptions are intended to match the logic used in Conley-style spatial econometrics. Researchers typically do not require global independence; they require that dependence be predominantly local and that long-range covariance accumulation be weak enough for CLTs and HAC estimators to stabilize~\citep{conley1999gmm,bester2011cluster,jenish2009clt}. The overlap requirement is the standard positivity condition from missing data analysis, interpreted here as a design requirement for audit allocation: every covariate-location stratum that contributes to the target mean must have nonzero probability of manual labeling.

The limitations are also clear. First, if overlap for the global mean is badly violated, valid confidence intervals are generally impossible without changing the estimand or adding structure; overlap-restricted estimands, adaptive trimming, or partial identification become appropriate alternatives~\citep{crump2009overlap}. Second, when dependence coordinates are proxy embeddings rather than true spatial coordinates, both HAC weighting and Moran-style diagnostics can be misaligned with the actual dependence structure. Third, the jackknife-HAC correction is designed for the practical regime of small labeled fractions and cross-fitted nuisance estimation; a complete theory for its finite-sample behavior under spatial dependence is beyond the current scope and is a natural target for future work.

Finally, this framework is modular: the same DR score can be paired with other dependence-appropriate covariance structures beyond spatial HAC. Appendix~\ref{app:modular} illustrates this with a jackknife-corrected two-way clustered variance estimator for overlapping dependence structures~\citep{cameron2011robust,chandrasekhar2023affinity}.

\section{Conclusions} 

We propose a method for spatially robust inference with predicted and missing at random labels. With cross-fitted nuisance estimation in the analytic sample, the standard DR estimator inherits fold-shared training noise, and naive Conley HAC variance estimation can misinterpret this noise as spatial dependence. We proposed a jackknife-HAC variance estimator with standard DR estimator. Under standard identification, dependent-array CLT conditions, and high-level nuisance and variance consistency assumptions, the resulting intervals are asymptotically valid.

Empirically, the jackknife-HAC correction can be evaluated across label budgets, fold choices, and coordinate quality. Theoretically, a natural direction is deriving the conditions under which this jackknife-HAC approach is asymptotically equivalent to HAC, while improving finite-sample robustness when labels are limited. More broadly, this approach is modular and can be paired with other dependence-matched covariance structures (e.g., Newey-West for time series~\citep{newey1987simple} or multiway clustering for overlapping dependence~\citep{cameron2011robust,chandrasekhar2023affinity}). These results suggest a broader framework for dependence-robust inference when machine learning predictions supplement sparse labels.

\newpage

\bibliography{spatial_post_prediction}

\clearpage

\appendix

\section*{Supplementary Material}

This appendix reports supplementary proofs, tables, and figures that support the main text. 

\section{Proof of Doubly Robust Identification}
\label{sec:dr_identification}

\begin{proof}
Condition on $(W_i,s_i)$ and write $m_i=m(W_i,s_i)$, $m_{0i}=m_0(W_i,s_i)$, $\pi_{0i}=\pi_0(W_i,s_i)$. For the first claim,
\begin{align*}
\E\!\left[\frac{R_i}{\pi_{0i}}(Y_i-m_i) \middle| W_i,s_i\right]
&=\frac{1}{\pi_{0i}}\E\!\left[R_i(Y_i-m_i)\mid W_i,s_i\right]\\
&=\E[Y_i-m_i\mid W_i,s_i],
\end{align*}
where MAR is used in the second equality. Hence
\[
\E\!\left[
m_i+\frac{R_i}{\pi_{0i}}(Y_i-m_i)
 \middle| 
W_i,s_i
\right]
=\E[Y_i\mid W_i,s_i],
\]
and averaging gives $\E[\psi_i(\theta_0;m,\pi_0)]=0$. For the second claim, $\E[Y_i-m_{0i}\mid W_i,s_i]=0$, so
\[
\E\!\left[\frac{R_i}{\pi(W_i,s_i)}(Y_i-m_{0i}) \middle| W_i,s_i\right]=0
\]
for any admissible $\pi$, and therefore $\E[\psi_i(\theta_0;m_0,\pi)]=0$.
\end{proof}

\section{Details on the True Labeling Propensity in Our Experiments}
\label{sec:propensity_dgm}

\paragraph{Synthetic MAR Propensity.}
In the synthetic full-scale experiments, labels are generated under
\[
R_i \mid (W_i,s_i,\widehat Y_i)\sim \mathrm{Bernoulli}\!\left(\pi_i^{\mathrm{MAR}}\right),
\]
with
\[
\pi_i^{\mathrm{MAR}}
=
\mathrm{clip}\!\left(
\operatorname{expit}(\alpha + S_i), 0.10, 0.90
\right),
\]
where $\alpha$ is chosen (by bisection) so that
\[
\frac{1}{n}\sum_{i=1}^n \pi_i^{\mathrm{MAR}} = 0.20.
\]
Define $z_i^{(f)}$, $z_i^{(c)}$, and $z_i^{(s)}$ to be the normalized features, $(W_i,s_i)$, coordinates, $s_i$, and predictions, $\widehat Y_i$, respectively. In the adversarial profile with strength $1.5$,
\[
\begin{aligned}
S_i={}&1.425 z^{(f)}_{i1}
+1.125 z^{(f)}_{i2}
+0.525 z^{(f)}_{i3}
+0.825 z^{(c)}_{i1}
-0.825 z^{(c)}_{i2} \\
&+0.600 z^{(f)}_{i1}z^{(c)}_{i1}
+0.525 z^{(c)}_{i1}z^{(c)}_{i2}
+0.450 z^{(f)}_{i2}z^{(c)}_{i2}
+1.350 z^{(s)}_{i}.
\end{aligned}
\]

\paragraph{How this is used in the experiments.}
The expression above is used at the \emph{label-assignment stage} of each synthetic replicate to generate the MAR mask  $R_i $. Specifically, after drawing a sample, we compute $\pi_i^{\mathrm{MAR}}$, draw $R_i\sim\mathrm{Bernoulli}(\pi_i^{\mathrm{MAR}})$, and enforce the label-budget bounds used by the simulation protocol. In our main simulation, we set
\[
\hat\pi_i=\mathrm{clip}\!\left(\pi_i^{\mathrm{MAR}}, 0.10, 0.90\right)
\]
in MAR arms; thus this is both the data-generating propensity and the oracle weighting propensity in those runs. (In MCAR arms, the oracle branch uses the constant target label rate.)

\section{Auxiliary-Split Prediction Design: Coverage and Width Tables}
\label{app:baux_tables}
These tables complement Table~\ref{tab:real-baux-main-coverage} in the main text by reporting how dependence-aware DR intervals change with the number of cross-fitting folds and with the critical-value choice used in the jackknife-HAC correction. Table~\ref{tab:real-baux-kimpact-coverage} reports observed-data coverage for Spatial DR-HAC and Spatial DR-JK-HAC at $K\in\{5,10\}$, and Table~\ref{tab:real-baux-kimpact-width} reports the corresponding mean interval lengths. Because the $K=10$ runs use fewer replicates than $K=5$, these fold-sensitivity comparisons are best read directionally.
\begin{table}[h]
\centering
\caption{Observed-data auxiliary-split prediction design: superpopulation-style coverage, fold-count sensitivity for dependence-aware methods.}
\label{tab:real-baux-kimpact-coverage}
\setlength{\tabcolsep}{4.0pt}
\scriptsize
\resizebox{\linewidth}{!}{%
\begin{tabular}{lrrrr}
\toprule
Dataset & Spatial DR-HAC & Spatial DR-JK-HAC (K=5, \(t\)) & Spatial DR-JK-HAC (K=10, \(t\)) & Spatial DR-JK-HAC (K=10, \(z\)) \\
\midrule
\multicolumn{5}{l}{\textit{Panel: MAR}} \\
Forest & 0.825 & 0.925 & 0.838 & 0.819 \\
Galaxies & 0.825 & 0.875 & 0.838 & 0.825 \\
Census income & 0.744 & 0.806 & 0.787 & 0.769 \\
Malaria & 0.750 & 0.850 & 0.819 & 0.794 \\
Health+ & 0.938 & 0.912 & 0.931 & 0.881 \\
\midrule
\multicolumn{5}{l}{\textit{Panel: MCAR}} \\
Forest & 0.894 & 0.969 & 0.944 & 0.925 \\
Galaxies & 0.906 & 0.944 & 0.925 & 0.900 \\
Census income & 0.875 & 0.906 & 0.900 & 0.869 \\
Malaria & 0.906 & 0.944 & 0.919 & 0.900 \\
Health+ & 0.944 & 0.900 & 0.912 & 0.881 \\
\bottomrule
\end{tabular}
}
\vspace{2pt}
\begin{minipage}{0.98\linewidth}
\footnotesize K=5 runs use 40 replicates per dataset-arm-sampling cell; K=10 runs use 20 replicates per cell.
\end{minipage}
\end{table}

\begin{table}[t]
\centering
\caption{Observed-data auxiliary-split prediction design: mean confidence-interval width by dataset and arm (nominal 90\%).}
\label{tab:real-baux-kimpact-width}
\setlength{\tabcolsep}{4.0pt}
\scriptsize
\resizebox{\linewidth}{!}{%
\begin{tabular}{lrrrr}
\toprule
Dataset & Spatial DR-HAC & Spatial DR-JK-HAC (K=5, \(t\)) & Spatial DR-JK-HAC (K=10, \(t\)) & Spatial DR-JK-HAC (K=10, \(z\)) \\
\midrule
\multicolumn{5}{l}{\textit{Panel: MAR}} \\
Forest & 0.160 & 0.232 & 0.184 & 0.165 \\
Galaxies & 0.214 & 0.314 & 0.255 & 0.229 \\
Census income & 30273.567 & 41671.111 & 32924.730 & 29543.385 \\
Malaria & 0.823 & 1.140 & 0.931 & 0.847 \\
Health+ & 11.337 & 15.016 & 11.098 & 9.984 \\
\midrule
\multicolumn{5}{l}{\textit{Panel: MCAR}} \\
Forest & 0.161 & 0.231 & 0.190 & 0.171 \\
Galaxies & 0.207 & 0.297 & 0.242 & 0.217 \\
Census income & 34798.853 & 47585.294 & 37588.786 & 33787.528 \\
Malaria & 1.016 & 1.404 & 1.157 & 1.047 \\
Health+ & 8.161 & 9.461 & 8.123 & 7.362 \\
\bottomrule
\end{tabular}
}
\vspace{2pt}
\begin{minipage}{0.98\linewidth}
\footnotesize Widths are averaged over IID and soft-block sampling within each dataset-arm cell.
\end{minipage}
\end{table}

\section{Synthetic Interval Width Diagnostics}
\label{app:synthwidth}
Figure~\ref{fig:synthwidth_box} reports per-population distributions of confidence-interval length for the synthetic simulation in Figure~\ref{fig:synth_cov_box}. We keep this figure in the appendix because width is secondary to coverage in the main-text narrative, but it remains useful for visualizing conservativeness tradeoffs across covariances. Figure~\ref{fig:synthcov_allmethods} reports the same synthetic coverage display with the full method menu, including DR-iid and spatial DR-HAC, which are omitted from the main-text coverage figure because they are either under-covering (DR-iid under MAR) or more conservative than spatial DR-JK-HAC under the main specification.

\begin{figure}[!ht]
\centering
\includegraphics[width=\linewidth]{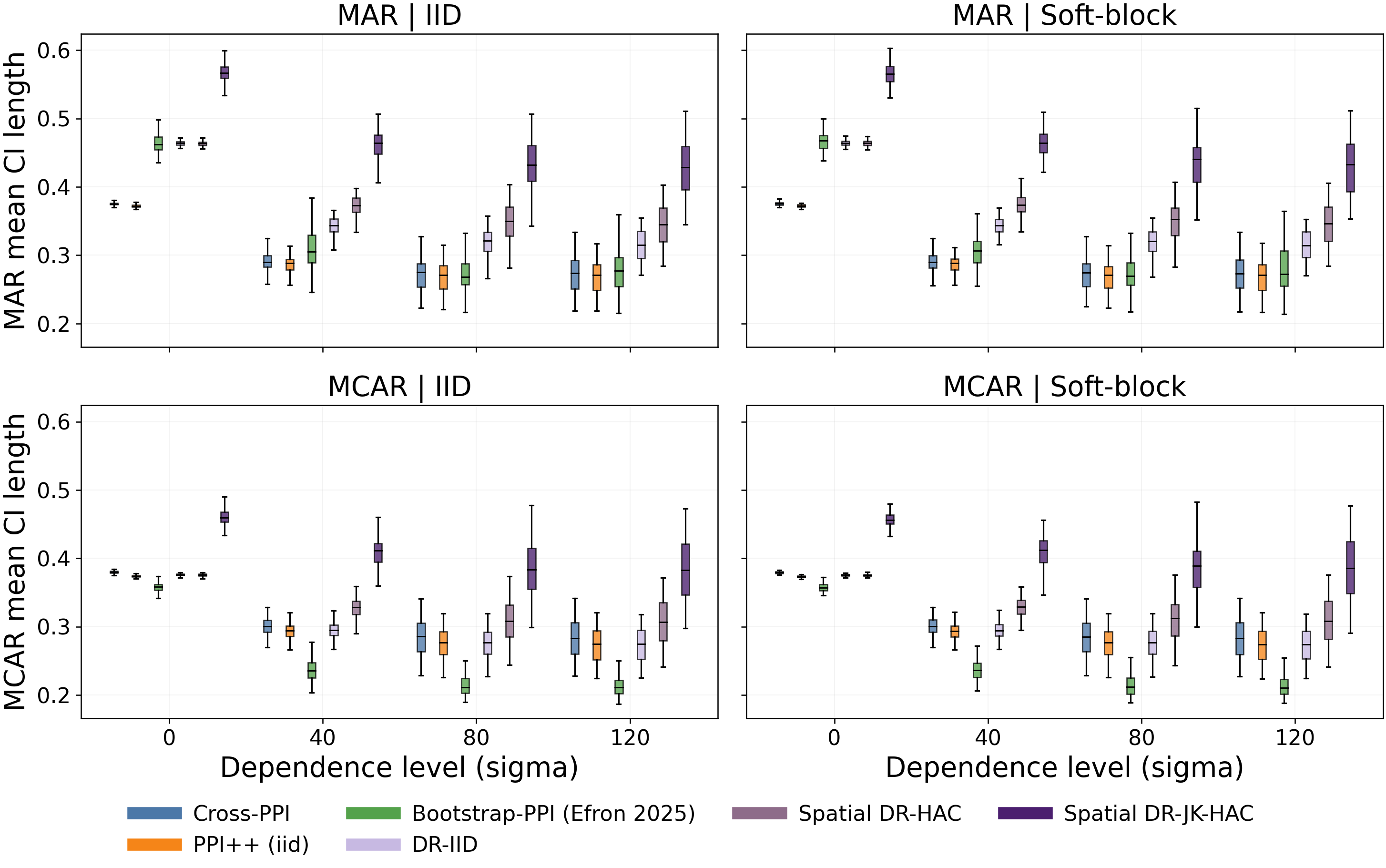}
\caption{Synthetic simulation (superpopulation target): confidence-interval length boxplots by missingness arm and sampling regime. The $x$-axis is $\sigma\in\{0,40,80,120\}$. Each box summarizes mean interval length over 200 repeated sample draws across 100 replicate populations. Purple shades denote DR-based methods (DR-iid, spatial DR-HAC, spatial DR-JK-HAC). The figure is self-contained: the dependence-aware methods inflate width relative to iid post-prediction baselines, and spatial DR-JK-HAC narrows intervals relative to spatial DR-HAC while retaining dependence-aware coverage.}
\label{fig:synthwidth_box}
\end{figure}

\begin{figure}[!ht]
\centering
\includegraphics[width=\linewidth]{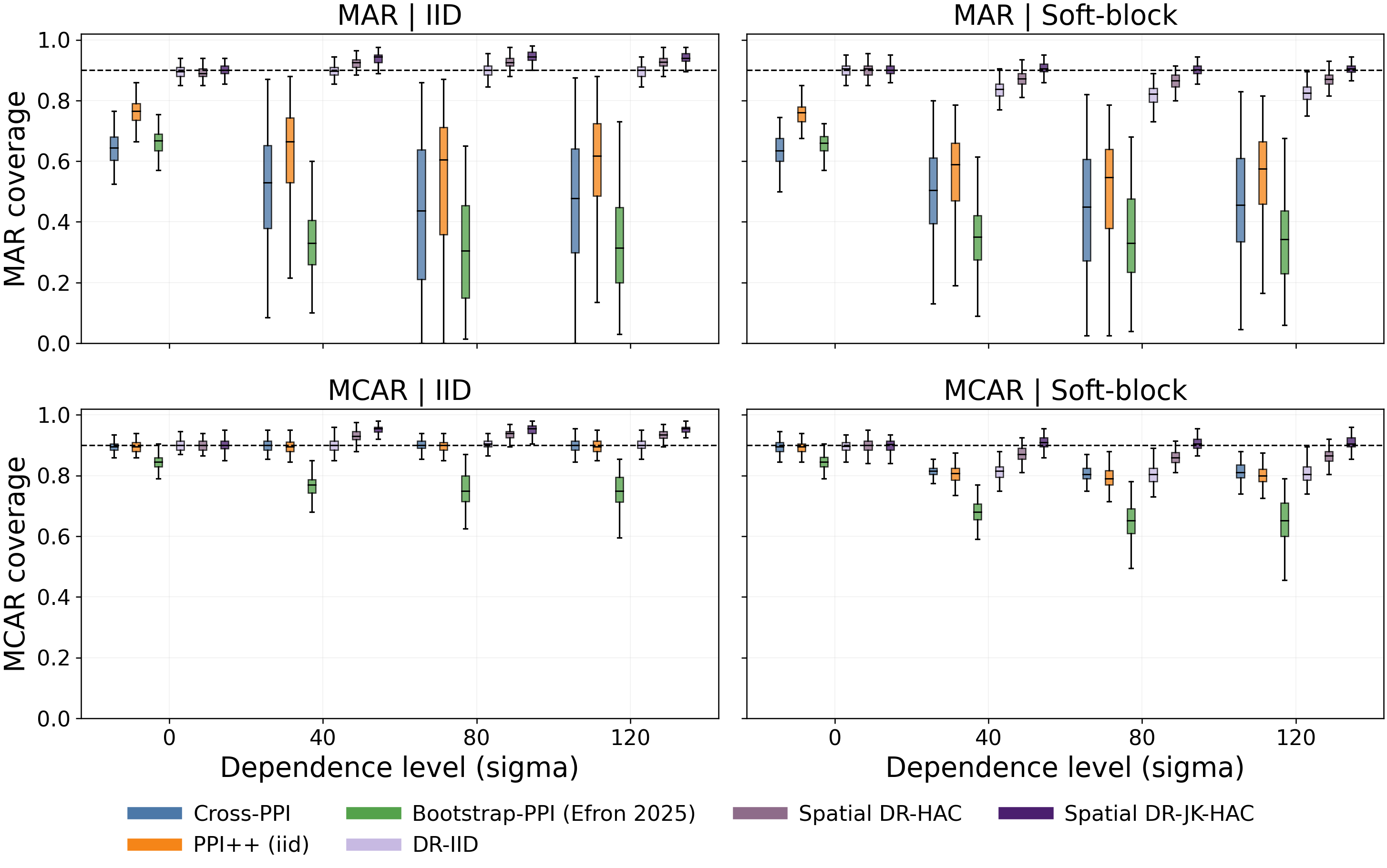}
\caption{Synthetic simulation: coverage boxplots for the full method menu. Methods include Cross-PPI, PPI++ (iid), Bootstrap-PPI (Efron 2025), DR-iid, spatial DR-HAC, and spatial DR-JK-HAC. The figure is self-contained: DR-iid and spatial DR-HAC improve over iid post-prediction baselines under dependence and MAR but do not match spatial DR-JK-HAC in the most challenging soft-block MAR cells, motivating the jackknife-HAC correction in the main text.}
\label{fig:synthcov_allmethods}
\end{figure}

\section{Additional Interval Visualizations}
\label{app:ci_examples}
Figures~\ref{fig:marcis} and \ref{fig:mcarcis} provide a by-eye check that complements the boxplot coverage summaries in Figure~\ref{fig:synth_cov_box}. Each panel shows 50 replicate-level 90\% intervals under {\it soft-block} sampling and 20\% labeling. The synthetic panel uses $\sigma=120$; the observed-data panels use the same auxiliary-split prediction construction and repeated-sampling protocol as the main empirical section. To keep the reference comparable across replicate-specific targets, we plot each interval {\it centered by its replicate target}, so the dashed vertical line at $0$ corresponds to the target mean in every panel. Blue intervals are Cross-PPI and purple intervals are Spatial DR-JK-HAC.

\begin{figure}[!ht]
\centering
\includegraphics[width=\linewidth]{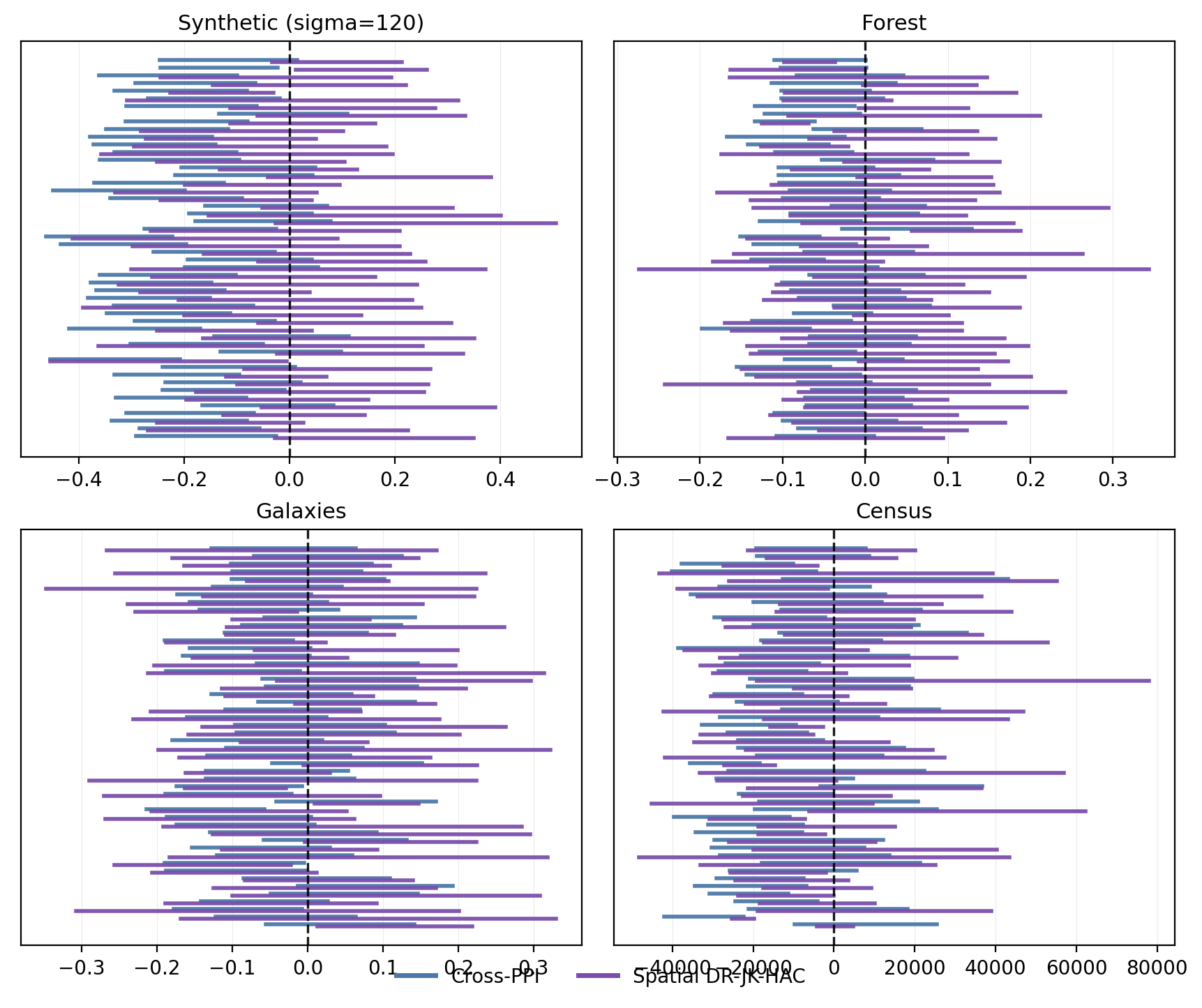}
\caption{Replicate-level interval visualization under MAR and soft-block sampling. Each interval is centered by subtracting its replicate target mean, so the dashed vertical line at $0$ is the target in every panel. Cross-PPI (blue) shows systematic mis-centering in multiple panels, consistent with the MAR under-coverage seen in Figure~\ref{fig:synth_cov_box}. Spatial DR-JK-HAC (purple) remains centered and attains near-nominal coverage.}
\label{fig:marcis}
\end{figure}

\begin{figure}[!ht]
\centering
\includegraphics[width=\linewidth]{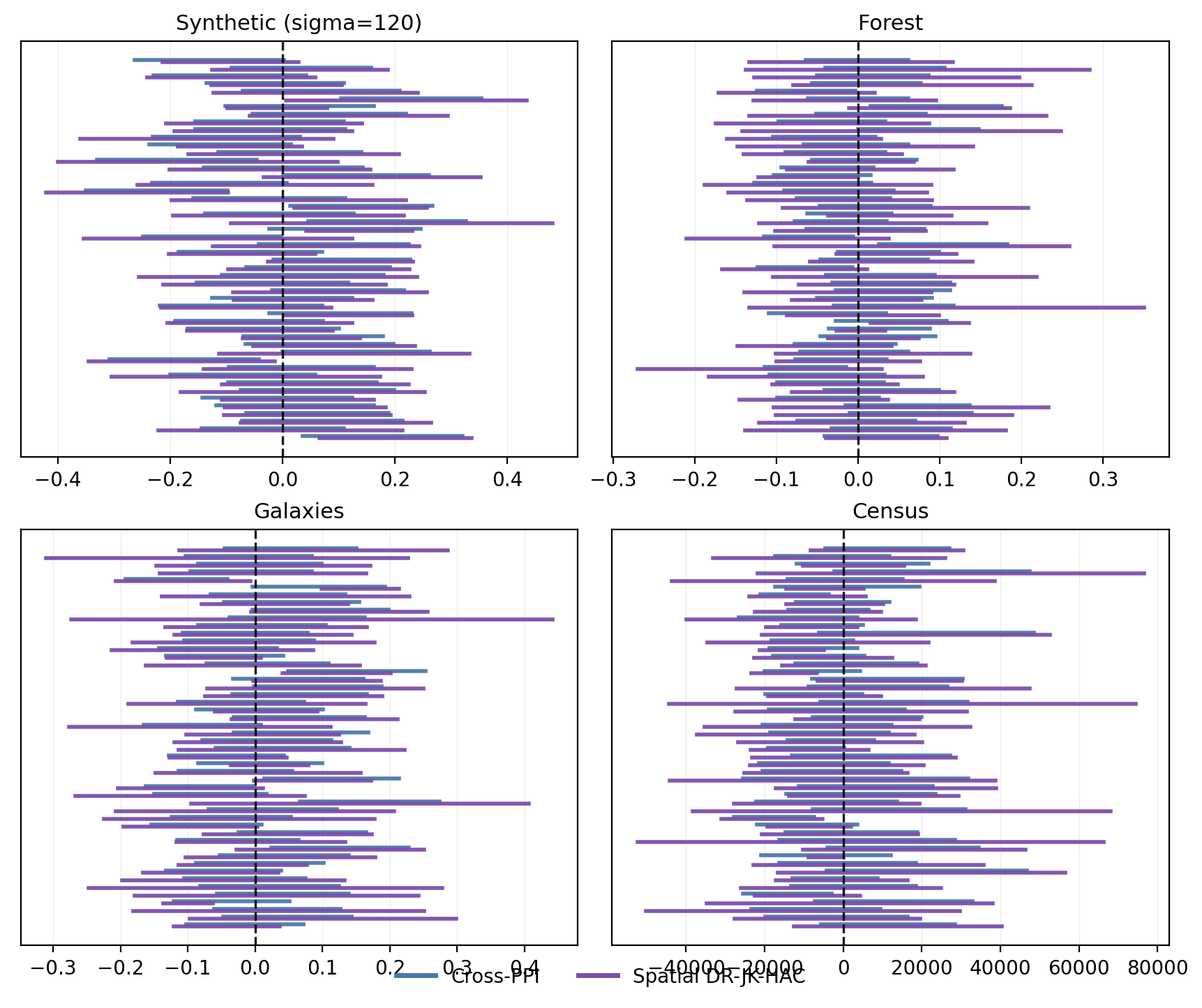}
\caption{Replicate-level interval visualization under MCAR and soft-block sampling, with the same centering convention as Figure~\ref{fig:marcis}. Under MCAR, iid baselines are often closer to nominal than under MAR, but Cross-PPI still under-covers in the dependence-rich synthetic panel as dependence becomes visible under soft-block sampling. Spatial DR-JK-HAC remains near nominal across panels at the cost of wider intervals.}
\label{fig:mcarcis}
\end{figure}

\section{Proxy Coordinates and Overlap Diagnostics}
\label{app:proxy_overlap}
This section addresses two empirical identification checks that are central to the main assumptions. First, for datasets that rely on proxy coordinates, we check whether residual dependence is actually visible in the coordinates used by inference. Second, we report overlap diagnostics directly rather than assuming them from design.

Table~\ref{tab:proxydiag} reports residual Moran's $I$ and nearest-neighbor residual correlation with permutation p-values. Forest and malaria show strong residual spatial signal under the analysis coordinates. Health+ shows significant global Moran's $I$ with weaker nearest-neighbor signal. Galaxies and census are weak under their proxy embeddings, which is exactly why we treat those settings as harder coordinate-proxy cases in the text.

\begin{table}[!ht]
\centering
\caption{Residual dependence diagnostics on analysis coordinates (5,000 permutation draws).}
\label{tab:proxydiag}
\small
\begin{tabular}{lcccc}
\toprule
Dataset & Moran's $I$ & Moran p-value & NN residual corr. & NN p-value \\
\midrule
Forest & 0.146 & 0.003 & 0.134 & 0.003 \\
Galaxies & -0.010 & 0.355 & 0.022 & 0.326 \\
Census income & 0.015 & 0.110 & 0.040 & 0.060 \\
Malaria & 0.293 & 0.003 & 0.352 & 0.003 \\
Health+ & 0.051 & 0.003 & 0.026 & 0.252 \\
\bottomrule
\end{tabular}
\end{table}

\begin{figure}[!ht]
\centering
\includegraphics[width=\linewidth]{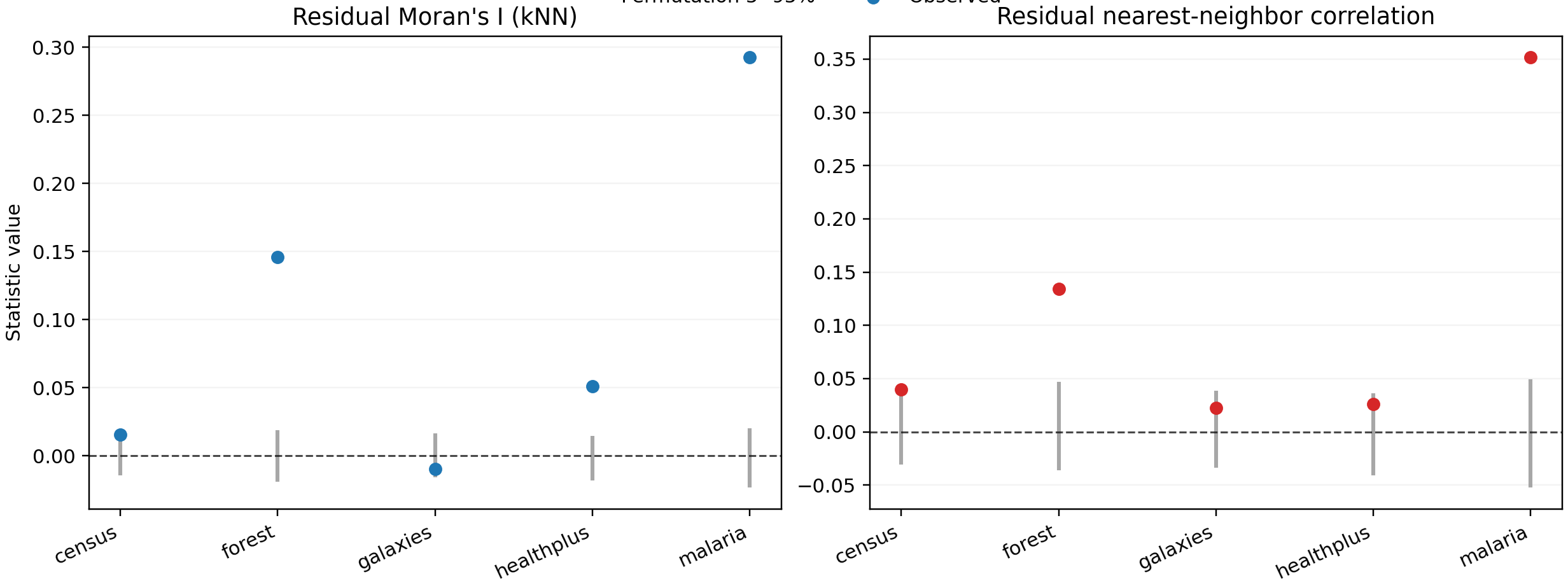}
\caption{Residual dependence diagnostics by dataset under the exact coordinates used in inference. Left panel reports Moran's $I$ with permutation intervals and right panel reports nearest-neighbor residual correlation with the same permutation reference. The figure is self-contained: forest and malaria are clearly dependence-rich under observed coordinates, while galaxies and census are weak under proxy embeddings.}
\label{fig:proxydiag}
\end{figure}

\begin{figure}[!ht]
\centering
\includegraphics[width=\linewidth]{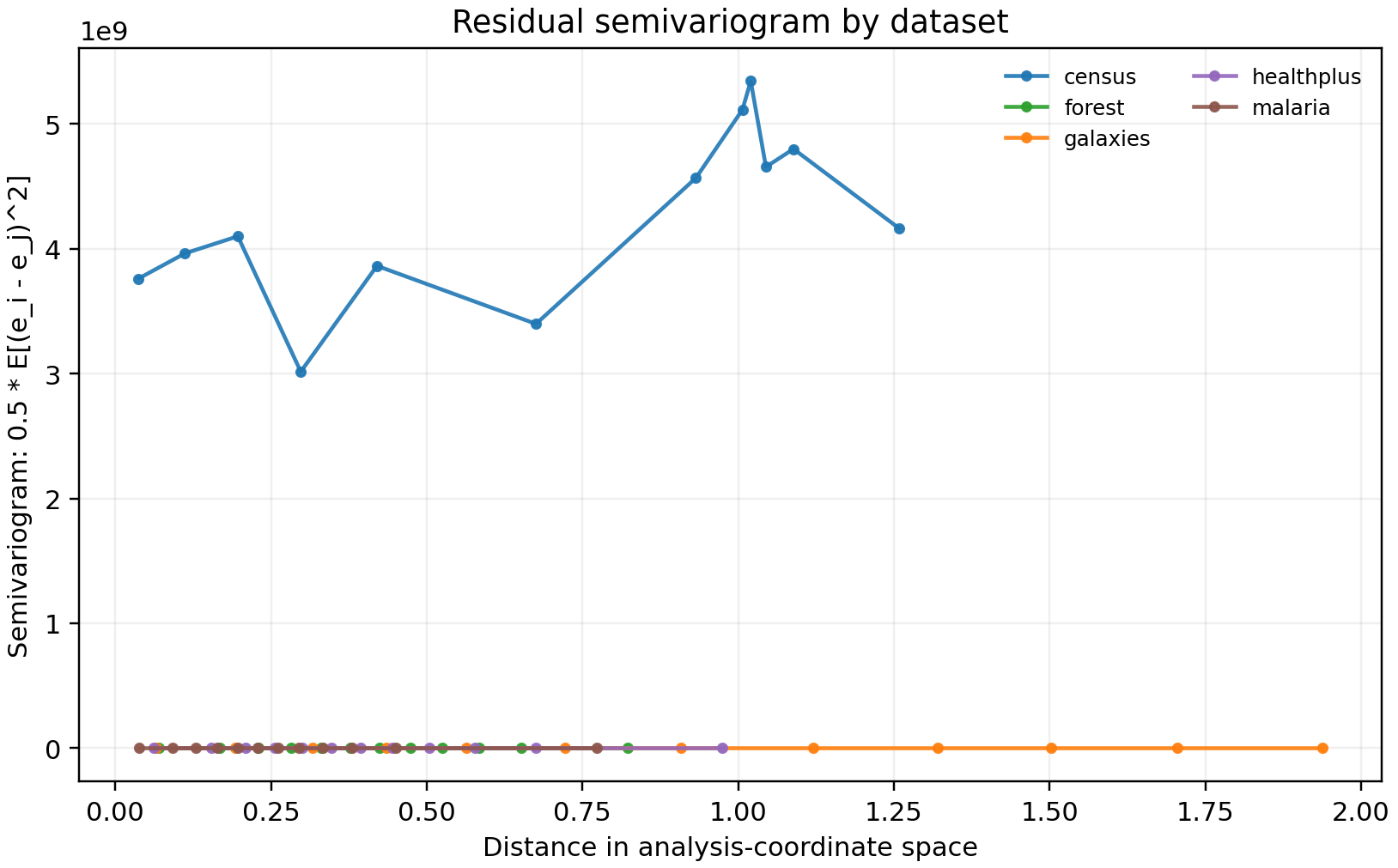}
\caption{Residual semivariograms for each empirical dataset under analysis coordinates. Rising semivariograms indicate stronger local dependence. Forest and malaria show the clearest distance-gradient pattern.}
\label{fig:proxysemi}
\end{figure}

For overlap, the design MAR floor is $0.10$ and the MCAR benchmark is $0.20$. Figure~\ref{fig:overlappi} reports the realized design and estimated propensity lower-tail quantiles under the exact sampling and MAR mechanism used in the main empirical analysis. The estimated propensities respect the declared overlap floor due to clipping, and MAR arms exhibit heavier lower-tail mass than MCAR arms. Figure~\ref{fig:overlapess} reports the corresponding inverse-propensity effective-sample-size ratios (ESS / labeled count), which quantify finite-sample information loss from nonuniform labeling probabilities.

\begin{figure}[!ht]
\centering
\includegraphics[width=\linewidth]{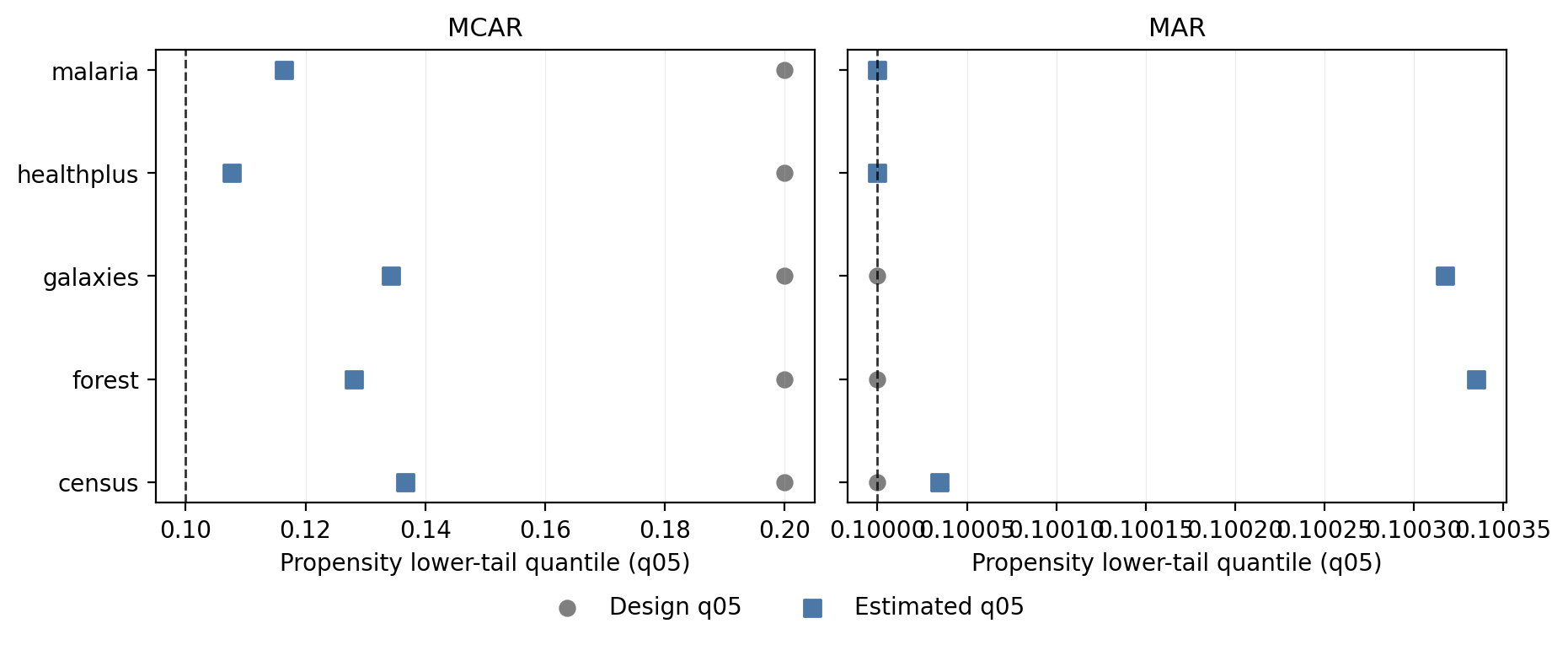}
\caption{Overlap diagnostics under the auxiliary-split observed-data design: design and estimated propensity lower-tail quantiles ($q05$) by dataset and arm, averaged over sampling regimes. The dashed line at $0.10$ marks the declared overlap floor.}
\label{fig:overlappi}
\end{figure}

\begin{figure}[!ht]
\centering
\includegraphics[width=\linewidth]{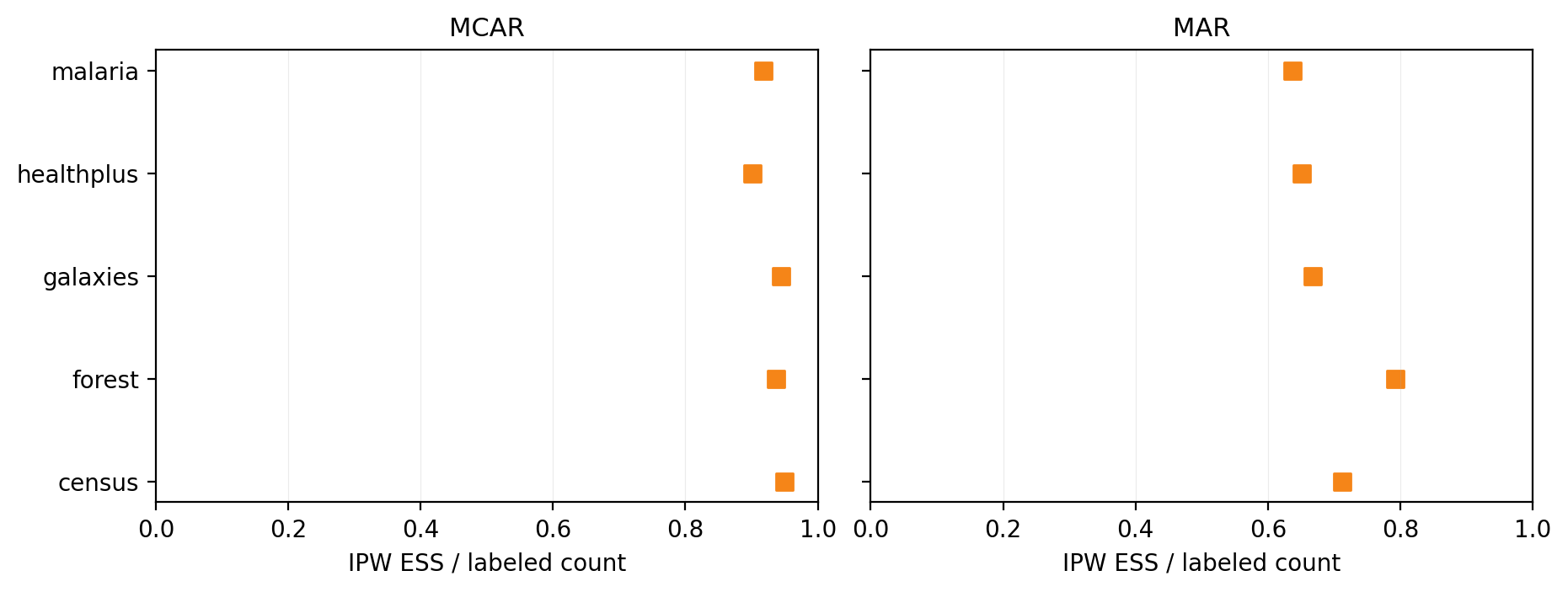}
\caption{Inverse-propensity effective-sample-size ratio (ESS / labeled count) under the auxiliary-split observed-data design, averaged over sampling regimes. MAR arms have lower ESS ratios, quantifying finite-sample information loss from nonuniform labeling probabilities.}
\label{fig:overlapess}
\end{figure}

\section{Mechanism Attribution}
\label{app:ablate_tune}
This appendix isolates which components are responsible for the calibration gains in the main experiments. We focus on the main synthetic design used for Figure~\ref{fig:synth_cov_box} and compare three nested DR variants that share the same AIPW point estimator but differ only in their covariance/critical-value construction: DR-iid, Spatial DR-HAC, and Spatial DR-JK-HAC. We also include the iid post-prediction baselines (Cross-PPI, PPI++, and Bootstrap-PPI) for reference. Because the point estimator is held fixed within the DR family, differences among these DR variants are directly attributable to dependence-aware variance estimation and to the jackknife-HAC correction that removes fold-shared nuisance noise.

Table~\ref{tab:ablation_hardcell} reports superpopulation coverage and mean interval length in the hardest cell of the main synthetic design ($\sigma=120$, soft-block sampling, $n=600$, 20\% labels). The incremental pattern is clear. Moving from iid post-prediction baselines to DR reduces MAR bias but does not, by itself, fix dependence-driven under-coverage (DR-iid). Adding a direct spatial HAC correction inflates variance and improves coverage (Spatial DR-HAC), and the jackknife-HAC correction restores near-nominal coverage in the most challenging soft-block MAR cell (Spatial DR-JK-HAC) at the cost of wider intervals. The same ordering is visible across the full dependence grid in Figure~\ref{fig:synthcov_allmethods}.

\begin{table}[!ht]
\centering
\caption{Mechanism attribution in the hardest synthetic cell ($\sigma=120$, soft-block sampling, 20\% labels, nominal 90\% superpopulation target). Each entry reports empirical coverage and mean interval length over 100 populations and 200 draws per population (20,000 intervals per arm).}
\label{tab:ablation_hardcell}
\small
\begin{tabular}{lcccc}
\toprule
Method & MCAR cov. & MCAR width & MAR cov. & MAR width \\
\midrule
Cross-PPI & 0.813 & 0.283 & 0.463 & 0.273 \\
PPI++ (iid) & 0.800 & 0.273 & 0.541 & 0.269 \\
Bootstrap-PPI (Efron 2025) & 0.652 & 0.215 & 0.345 & 0.288 \\
DR-iid & 0.805 & 0.274 & 0.823 & 0.315 \\
Spatial DR-HAC & 0.864 & 0.308 & 0.869 & 0.346 \\
Spatial DR-JK-HAC & 0.908 & 0.385 & 0.903 & 0.430 \\
\bottomrule
\end{tabular}
\end{table}

\section{Modular Covariance Extension Beyond Spatial HAC}
\label{app:modular}
To make this appendix self-contained, we state the exact simulation setup. Units $i=1,\dots,n$ are assigned two overlapping dependence labels $c_1(i)\in\{1,\dots,G_1\}$ and $c_2(i)\in\{1,\dots,G_2\}$. Outcomes follow
\[
Y_i=\mu+\beta X_i+U_{c_1(i)}+V_{c_2(i)}+\varepsilon_i,
\]
with $U_g$, $V_h$, and $\varepsilon_i$ mean-zero shocks. Labels are revealed under MCAR or MAR with overlap floors. We keep the same DR score as in the main text,
\[
\hat\theta=\frac{1}{n}\sum_{i=1}^n\left[\hat m_i+\frac{R_i}{\hat\pi_i}(Y_i-\hat m_i)\right].
\]
In this setup, $\hat m_i$ is the external prediction $ \hat Y_i $, and $\hat\pi_i$ is estimated by $K=5$ fold-wise cross-fitted logistic regression. This induces the same fold-level shared-nuisance artifact as in the spatial runs.

Let $\hat\psi_i$ denote estimated scores. For fold $k(i)$, define $\tilde\psi_i=\hat\psi_i-\bar\psi_{k(i)}$. Denote the Cameron--Gelbach--Miller two-way covariance by $\widehat V_{\text{2way}}(\cdot)$. The jackknife-corrected modular estimator is
\[
\widehat V_{\text{JK-2way}}
=
\left(\widehat V_{\text{2way}}(\tilde\psi)-\frac{1}{n^2}\sum_{i=1}^n\tilde\psi_i^2\right)
\;+\;
\frac{K}{K-1}\sum_{k=1}^K\left(\frac{n_k}{n}\right)^2\left(\bar\psi_k-\hat\theta\right)^2.
\]
The first bracket isolates off-diagonal dependence after removing fold-shared nuisance noise; the second term restores between-fold macro variation. Confidence intervals are $\hat\theta\pm z_{1-\alpha/2}\sqrt{\widehat V_{\text{JK-2way}}}$.

The same construction maps naturally to network settings when analysis units have two overlapping incidence dimensions. For dyadic outcomes, each edge observation belongs to a sender cluster and a receiver cluster; for panel-network settings, observations can belong to node and time clusters. In each case the DR score remains unchanged and only $\hat V$ is swapped to the relevant multiway form. We therefore present this section as a modularity demonstration, not a tuned recommendation.

\paragraph{Algorithm H.1 (Jackknife Two-Way Covariance).}
For completeness, we state the exact estimator used in this appendix as an executable sequence.
\begin{enumerate}
\item Compute cross-fitted $\hat\pi_i$ over $K$ folds, clip to $[\underline\pi,1-\underline\pi]$, and set $\hat m_i=\hat Y_i$.
\item Build scores $\hat\psi_i=\hat m_i+\frac{R_i}{\hat\pi_i}(Y_i-\hat m_i)$, $\hat\theta=n^{-1}\sum_i\hat\psi_i$, and fold means $\bar\psi_k$.
\item Demean within fold: $\tilde\psi_i=\hat\psi_i-\bar\psi_{k(i)}$.
\item Compute $\widehat V_{\text{2way}}(\tilde\psi)=\widehat V_{g_1}(\tilde\psi)+\widehat V_{g_2}(\tilde\psi)-\widehat V_{g_1,g_2}(\tilde\psi)$ using Cameron--Gelbach--Miller inclusion-exclusion sums.
\item Subtract diagonal piece $n^{-2}\sum_i\tilde\psi_i^2$.
\item Add between-fold ANOVA term $\frac{K}{K-1}\sum_k (n_k/n)^2(\bar\psi_k-\hat\theta)^2$.
\item Report $\hat\theta\pm z_{1-\alpha/2}\sqrt{\widehat V_{\text{JK-2way}}}$.
\end{enumerate}
To make the appendix reproducible at implementation level, we include a fully annotated reference implementation in the supplementary code bundle matching Algorithm H.1.

Figure~\ref{fig:modcov90} reports nominal $90\%$ coverage as dependence strength increases. At the weakest dependence level ($\texttt{dep\_scale}=0$), MAR coverage is $0.533$ for Cross-PPI, $0.952$ for the two-way plug-in DR, and $0.853$ for the jackknife-corrected two-way DR. Thus the jackknife correction removes much of the weak-dependence over-coverage from the plug-in estimator while keeping interval width essentially unchanged ($0.143$ vs.\ $0.141$). As dependence strengthens, both two-way DR variants become very conservative, with near-identical width profiles (Figure~\ref{fig:modwidth90}), indicating that in this design the remaining inflation is dominated by genuine overlap dependence rather than fold-sharing noise. Full $80\%$ and $95\%$ versions are shown in Figure~\ref{fig:modcov80} and Figure~\ref{fig:modcov95}.

\begin{figure}[!ht]
\centering
\includegraphics[width=\linewidth]{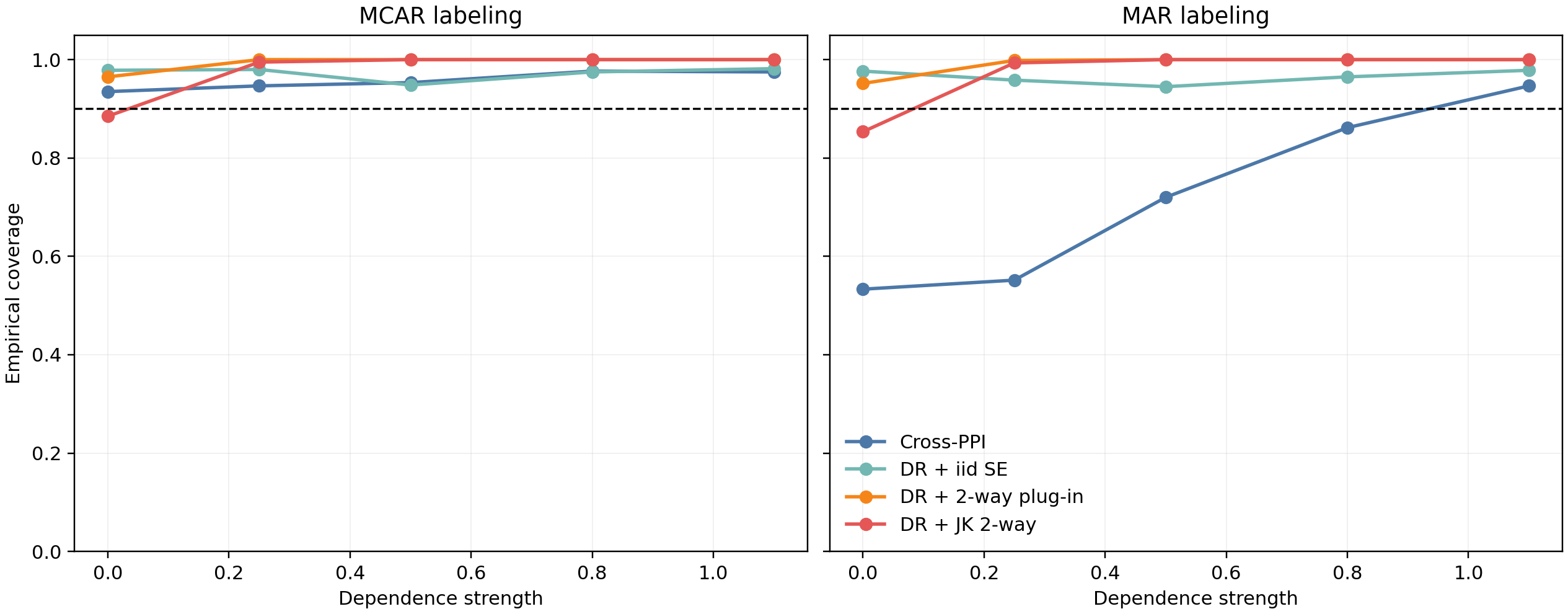}
\caption{Modular extension simulation with overlapping two-way dependence at nominal $90\%$. Cross-PPI remains anti-conservative in MAR. The jackknife two-way DR interval is less conservative than the direct two-way plug-in at weak dependence, while both methods are highly conservative in strong-overlap regimes.}
\label{fig:modcov90}
\end{figure}

\begin{figure}[!ht]
\centering
\includegraphics[width=\linewidth]{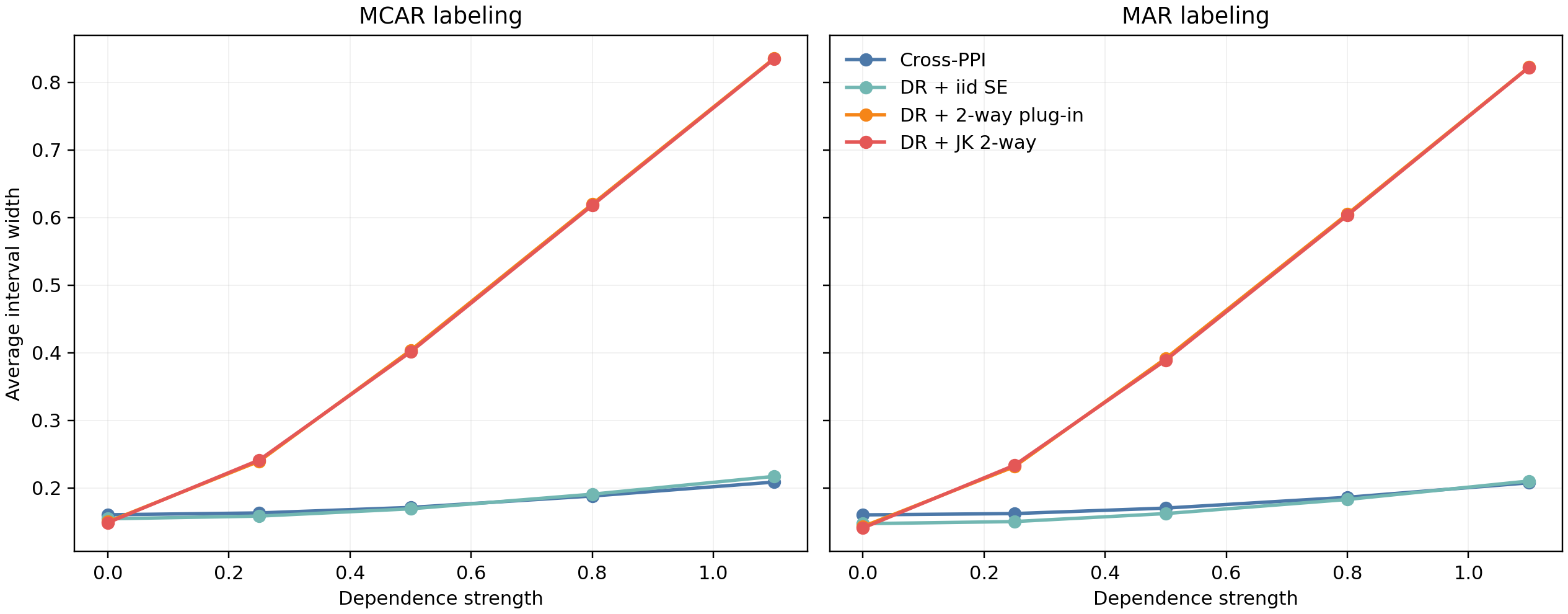}
\caption{Interval width in the modular extension simulation at nominal $90\%$. Jackknife and direct two-way plug-in widths are nearly identical throughout, and both increase sharply as overlap dependence strengthens.}
\label{fig:modwidth90}
\end{figure}

\begin{figure}[!ht]
\centering
\includegraphics[width=\linewidth]{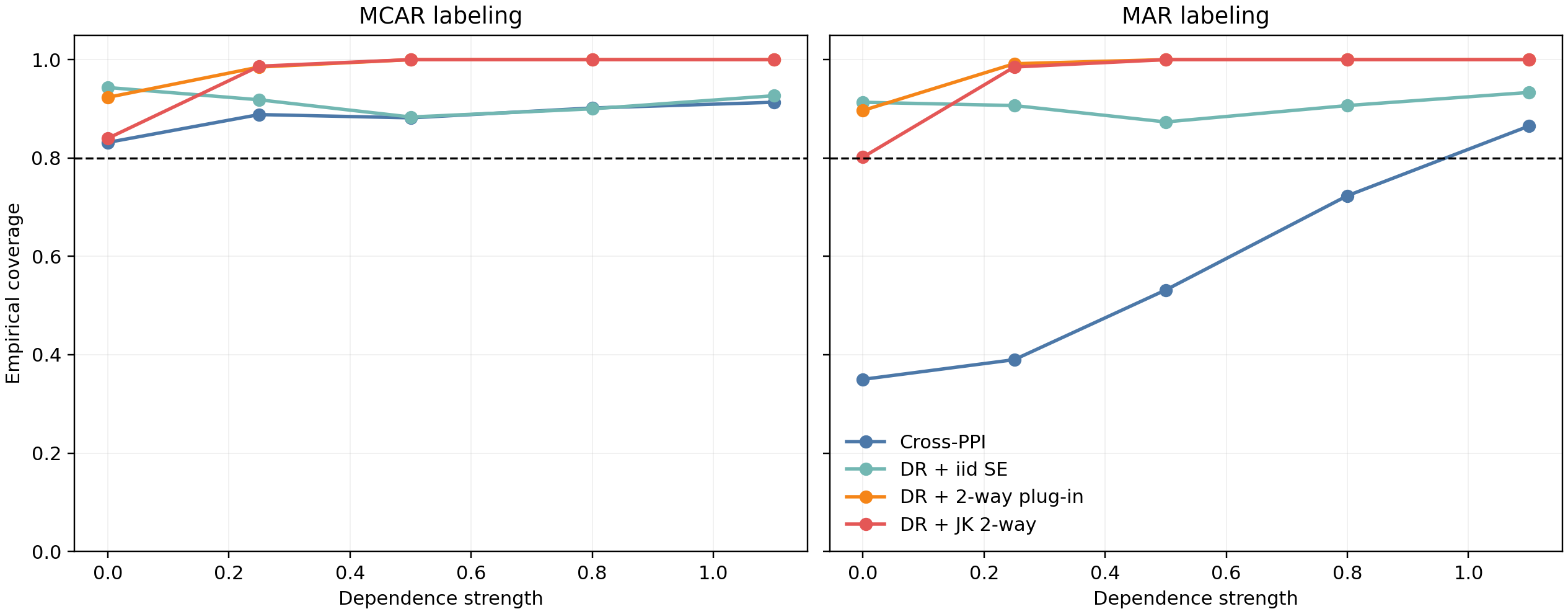}
\caption{Modular extension coverage at nominal $80\%$.}
\label{fig:modcov80}
\end{figure}

\begin{figure}[!ht]
\centering
\includegraphics[width=\linewidth]{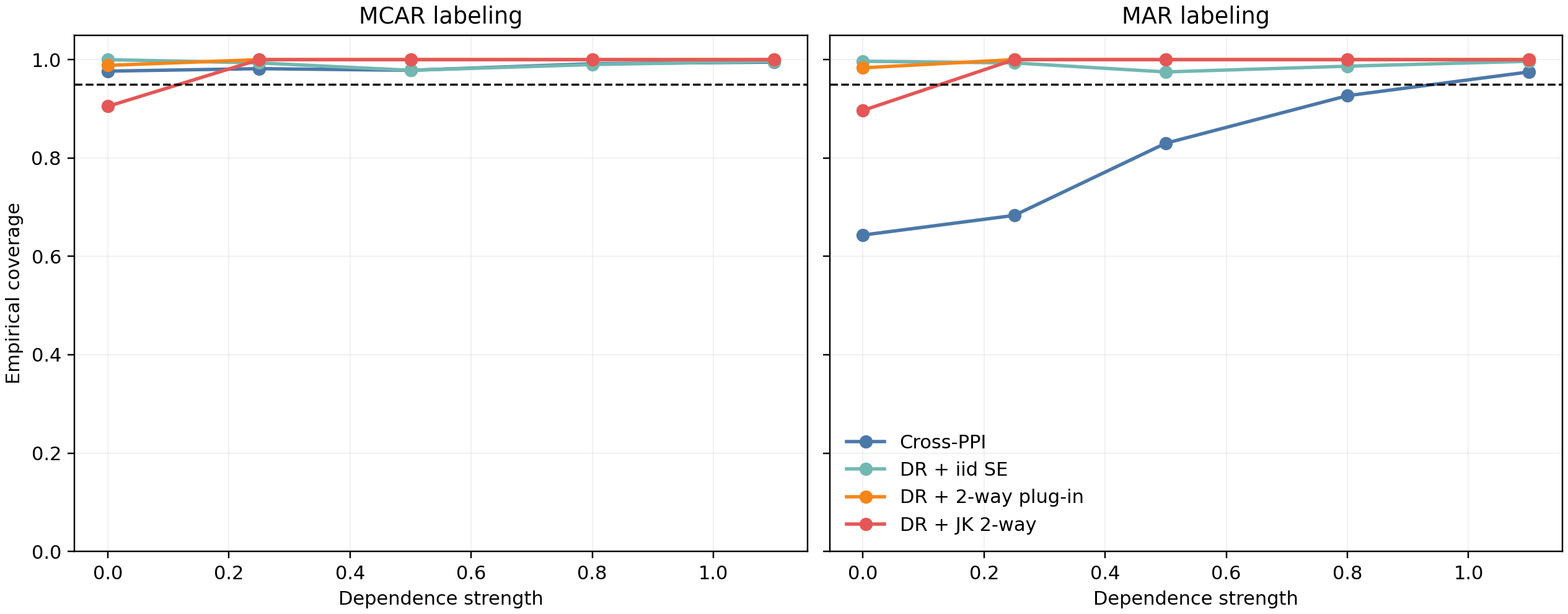}
\caption{Modular extension coverage at nominal $95\%$.}
\label{fig:modcov95}
\end{figure}

\section{Complete Annotated Proof of Theorem~\ref{thm:main}}
\label{app:proof}
This section gives a line-by-line derivation of Theorem~\ref{thm:main} and explicitly marks where each assumption is used.

\paragraph{Notation and decomposition.}
Define
\[
\phi_{0i}=\psi_i(\theta_0;m_0,\pi_0),\qquad
\hat\phi_i=\psi_i(\theta_0;\hat m,\hat\pi).
\]
Write $m_{0i}=m_0(W_i,s_i)$ and $\pi_{0i}=\pi_0(W_i,s_i)$ for unit-level oracle nuisances.
Because
\[
\hat\theta=\frac{1}{n}\sum_{i=1}^n\Big[\psi_i(\theta_0;\hat m,\hat\pi)+\theta_0\Big],
\]
we have
\[
\hat\theta-\theta_0=\frac{1}{n}\sum_{i=1}^n \hat\phi_i.
\]
Add and subtract $\phi_{0i}$:
\[
\sqrt{n}(\hat\theta-\theta_0)
=
\frac{1}{\sqrt{n}}\sum_{i=1}^n \phi_{0i}
\;+\;
\sqrt{n}R_n,
\quad
R_n=\frac{1}{n}\sum_{i=1}^n(\hat\phi_i-\phi_{0i}).
\]
So the theorem reduces to two tasks: show asymptotic normality of the oracle-score sum and show $\sqrt{n}R_n=o_p(1)$.

\paragraph{Step 1: Mean-zero oracle score (identification step).}
By Proposition~\ref{prop:dr} under Assumption~\ref{ass:mar},
\[
\E[\phi_{0i}]=0\quad\text{for each }i.
\]
Intuition: this is where MAR and overlap identify the target. Without this step, the leading term would be centered at a nonzero bias and no covariance correction could fix coverage.

\paragraph{Step 2: Centered triangular-array CLT term.}
Set
\[
Z_{n,i}=\phi_{0i}-\E[\phi_{0i}].
\]
From Step 1, $Z_{n,i}=\phi_{0i}$. Assumption~\ref{ass:clt} implies
\[
\frac{1}{\sqrt{n}}\sum_{i=1}^n Z_{n,i}\od \mathcal{N}(0,\sigma^2),
\qquad \sigma^2\in(0,\infty).
\]
Hence
\[
\frac{1}{\sqrt{n}}\sum_{i=1}^n \phi_{0i}\od \mathcal{N}(0,\sigma^2).
\]
Intuition: this is the dependence part of the theorem. Assumption~\ref{ass:clt} replaces iid CLT arguments with covariance-array controls tailored to local dependence.

\paragraph{Step 3: Nuisance remainder control.}
Assumption~\ref{ass:nuisance} states
\[
R_n=o_p(n^{-1/2}),
\]
so directly
\[
\sqrt{n}R_n=o_p(1).
\]
To make the DR structure explicit, expand
\[
R_n
=
\frac{1}{n}\sum_{i=1}^n
\Big(1-\frac{R_i}{\hat\pi_i}\Big)(\hat m_i-m_{0i})
\;+\;
\frac{1}{n}\sum_{i=1}^n
R_i(Y_i-m_{0i})\Big(\frac{1}{\hat\pi_i}-\frac{1}{\pi_{0i}}\Big).
\]
Under cross-fitting and overlap clipping, this remainder is small if either nuisance is first-order correct (propensity-consistent with square-integrable $\hat m$, or outcome-regression-consistent with clipped $\hat\pi$); the product-rate condition is the general misspecification-robust alternative \citep{bang2005doubly,chernozhukov2018dml}. This is the formal sense in which the asymptotic argument is doubly robust.
Intuition: orthogonality plus cross-fitting makes first-order nuisance error cancel, leaving only a second-order remainder. Buffered splitting is used so this cancellation is credible in dependent data, not only in iid designs.

\paragraph{Step 4: Combine Steps 2 and 3 (first theorem claim).}
From the decomposition,
\[
\sqrt{n}(\hat\theta-\theta_0)
\;=\;
\underbrace{\frac{1}{\sqrt{n}}\sum_{i=1}^n \phi_{0i}}_{\od\mathcal{N}(0,\sigma^2)}
\;+\;
\underbrace{\sqrt{n}R_n}_{o_p(1)}.
\]
Slutsky's theorem gives
\[
\sqrt{n}(\hat\theta-\theta_0)\od \mathcal{N}(0,\sigma^2).
\]
This proves asymptotic normality.

\paragraph{Step 5: Studentization (second theorem claim).}
Define
\[
T_n=\frac{\hat\theta-\theta_0}{\sqrt{\hat V_{\mathrm{JK}}^{+}}}.
\]
Assumption~\ref{ass:hac} gives $n\hat V_{\mathrm{JK}}^{+}\op\sigma^2$, hence $\sqrt{n\hat V_{\mathrm{JK}}^{+}}\op\sigma$. Combine with Step 4:
\[
T_n
=
\frac{\sqrt{n}(\hat\theta-\theta_0)}{\sqrt{n\hat V_{\mathrm{JK}}^{+}}}
\od \mathcal{N}(0,1).
\]
Therefore, for fixed $\alpha\in(0,1)$,
\[
\Pr\!\left(
\theta_0\in
\left[
\hat\theta-z_{1-\alpha/2}\sqrt{\hat V_{\mathrm{JK}}^{+}}, 
\hat\theta+z_{1-\alpha/2}\sqrt{\hat V_{\mathrm{JK}}^{+}}
\right]
\right)\to 1-\alpha.
\]
This proves asymptotic validity of the reported confidence interval.

\paragraph{Assumption-to-step map (sanity check).}
Assumption~\ref{ass:mar} enters only through Proposition~\ref{prop:dr} (Step 1), which fixes the estimand and removes first-order selection bias. Assumption~\ref{ass:clt} controls the dependence structure of the oracle score sum (Step 2). Assumption~\ref{ass:pred} governs whether prediction scores are externally valid or must be cross-fitted in-sample. Assumption~\ref{ass:nuisance} controls estimation error from $(\hat m,\hat\pi)$ (Step 3). Assumption~\ref{ass:hac} links the estimated covariance to the same limiting variance used in Step 2 (Step 5). If any one of these five links fails, nominal coverage is not guaranteed; this is exactly why the supplement reports out-of-scope stress tests.

\end{document}